\documentclass[10pt,twocolumn,letterpaper]{article}

\usepackage{cvpr}
\usepackage{times}
\usepackage{epsfig}
\usepackage{graphicx}
\usepackage{amsmath}
\usepackage{amssymb}
\usepackage{amsthm}
\usepackage{mathtools}

\usepackage{makecell}
\usepackage{booktabs}
\usepackage{floatrow}
\DeclareFloatFont{tiny}{\footnotesize}%
\usepackage{mdwlist}
\usepackage{nopageno}

\usepackage{capt-of,etoolbox}

\usepackage[tight,footnotesize,sf,SF]{subfigure}
\usepackage{pgfplots}
\pgfplotsset{compat=newest}

\tikzset{every picture/.style={font issue=\scriptsize},
         font issue/.style={execute at begin picture={#1\selectfont}}
        }

\usepackage[noend]{algorithm2e}

\newtheorem{mydef}{Definition}
\newtheorem{corollary}{Corollary}

\newtheorem{lemma}{Lemma}

\DeclareMathOperator{\diag}{diag}

\DeclareMathOperator{\VEC}{vec}

\DeclareMathOperator*{\argmin}{arg\,min}

\newcommand{\R}{\mathbb{R}}
\newcommand{\N}{\mathbb{N}}

\newcommand{\onevec}{\mathbf{1}}
\newcommand{\zerovec}{\mathbf{0}}
\newcommand{\matI}{\mathbf{I}}
\newcommand{\matX}{\mathbf{X}}

\newcommand{\matY}{\mathbf{Y}}
\newcommand{\matQ}{\mathbf{Q}}
\newcommand{\matPhi}{\mathbf{\Phi}}

\newcommand{\matD}{\mathbf{D}}

\newcommand{\matK}{\mathbf{K}}

\newcommand{\matS}{\mathbf{S}}
\newcommand{\matZ}{\mathbf{Z}}

\newcommand{\matE}{\mathbf{E}}
\newcommand{\vecphi}{\boldsymbol{\phi}}

\newcommand{\matA}{\mathbf{A}}

\newcommand{\vecb}{\mathbf{b}}
\newcommand{\vecv}{\mathbf{v}}
\newcommand{\vecalpha}{\boldsymbol\alpha}
\newcommand{\vecomega}{\boldsymbol\omega}
\newcommand{\matAlpha}{\boldsymbol A}

\newcommand{\veca}{\mathbf{a}}

\newcommand{\vecz}{\mathbf{z}}

\newcommand{\vecx}{\mathbf{x}}
\newcommand{\vecy}{\mathbf{y}}

\newcommand{\matC}{\mathbf{C}}

\newcommand{\prox}{\operatorname{prox}}
\newcommand{\proj}{\operatorname{proj}}
\usepackage[pagebackref=true,breaklinks=true,letterpaper=true,colorlinks,bookmarks=false]{hyperref}

\cvprfinalcopy %

\def\cvprPaperID{237} %

\ifcvprfinal\fi

\begin{document}

\setlength{\abovedisplayskip}{4.5pt}
\setlength{\belowdisplayskip}{4.5pt}
\setlength{\textfloatsep}{2mm}
\setlength{\dbltextfloatsep}{2mm}
\setlength{\dblfloatsep}{0mm}

\title{Linear Shape Deformation Models with Local Support using Graph-based Structured Matrix Factorisation}
\author{
Florian Bernard\textsuperscript{1,2}  $\quad$
Peter Gemmar\textsuperscript{2,3} 
 $\quad$
Frank Hertel\textsuperscript{1} 
 $\quad$
Jorge Goncalves\textsuperscript{2}
$\quad$
Johan Thunberg\textsuperscript{2} 
\\
\textsuperscript{1}Centre Hospitalier de Luxembourg, Luxembourg \\
\textsuperscript{2}Luxembourg Centre for Systems Biomedicine, University of Luxembourg, Luxembourg\\
\textsuperscript{3}Trier University of Applied Sciences, Trier, Germany
}

\makeatletter
\let\@oldmaketitle\@maketitle%
\renewcommand{\@maketitle}{\@oldmaketitle%
  \myfigure{}\bigskip}%
\makeatother

\newcommand\myfigure{%
\centering
    \vspace{-5mm}
    \includegraphics[scale=0.113]{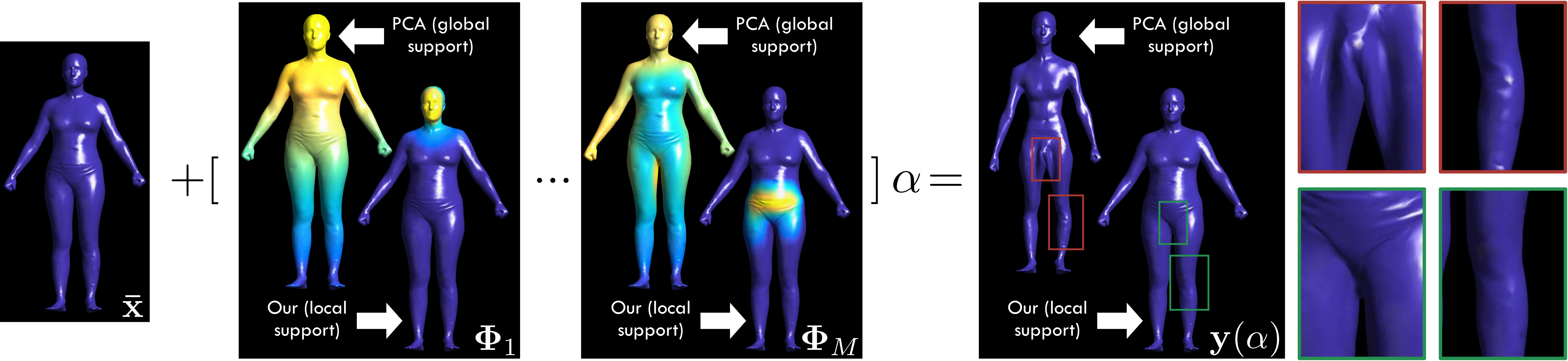}
    \captionof{figure}{Global support factors of PCA lead to implausible body shapes, whereas the local support factors of our method give more realistic results. See our accompanying video for animated results.}
     \label{globvsloc}
}

\maketitle

\begin{abstract}
Representing 3D shape deformations by high-dimensional linear models has many applications in computer vision and medical imaging. Commonly, using Principal Components Analysis a low-dimensional subspace of the high-dimensional shape space is determined. However, the resulting factors (the most dominant eigenvectors of the covariance matrix) have global support, i.e.~changing the coefficient of a single factor deforms the entire shape. Based on matrix factorisation with sparsity and graph-based regularisation terms, we present a method to obtain deformation factors with local support. The benefits include better flexibility and interpretability as well as the possibility of interactively deforming shapes locally. We demonstrate that for brain shapes our method outperforms the state of the art in local support models with respect to generalisation and sparse reconstruction, whereas for body shapes our method gives more realistic deformations.
\end{abstract}

\footnotetext{\copyright ~ 2015 IEEE. Personal use of this material is permitted. Permission from IEEE must be obtained for all other uses, in any current or future media, including reprinting/republishing this material for advertising or promotional purposes, creating new collective works, for resale or redistribution to servers or lists, or reuse of any copyrighted component of this work in other works.}

\section{Introduction}
Due to their simplicity, linear models in high-dimensional space are frequently used for modelling nonlinear deformations of shapes in 2D or 3D space. For example, Active Shape Models (ASM) \cite{Cootes:1992uw}, based on a statistical shape model, are popular for image segmentation.
Usually, surface meshes, comprising faces and vertices, are employed for representing the surfaces of shapes in 3D. Dimensionality reduction techniques are used to learn a low-dimensional representation of the \emph{vertex coordinates} from training data. %
Frequently, an affine subspace close to the training shapes is used. To be more specific, mesh deformations are modelled by expressing the vertex coordinates as the sum of a mean shape $\bar{\vecx}$ and a linear combination of $M$ modes of variation $\matPhi = [\matPhi_1,\ldots,\matPhi_M]$, i.e.~the vertices deformed by the weight or coefficient vector $\vecalpha$ are given by $\vecy(\vecalpha)=\bar{\vecx} + \matPhi \vecalpha$, see Fig.~\ref{globvsloc}.
Commonly, by using Principal Components Analysis (PCA), the modes of variation are set to the most dominant eigenvectors of the sample covariance matrix.
PCA-based models are computationally convenient due to the orthogonality of the eigenvectors of the (real and symmetric) covariance matrix. Due to the diagonalisation of the covariance matrix, an axis-aligned Gaussian distribution of the weight vectors of the training data is obtained.
A problem of PCA-based models is that eigenvectors have \emph{global support}, i.e.~adjusting the weight of a single factor affects \emph{all} vertices of the shape (Fig.~\ref{globvsloc}). 

Thus, in this work, instead of eigenvectors, we consider more general \emph{factors} as modes of variation that have \emph{local support}, i.e.~adjusting the weight of a single factor leads only to a spatially localised deformation of the shape (Fig.~\ref{globvsloc}). The set of all factors can be seen as a dictionary for representing shapes by a linear combination of the factors.
Benefits of factors with local support include more realistic deformations (cf.~Fig.~\ref{globvsloc}), better interpretability of the deformations (e.g. clinical interpretability in a medical context \cite{Sjostrand:2007hb}), and the possibility of interactive local mesh deformations (e.g. editing animated mesh sequences in computer graphics \cite{Neumann:2013gb}, or enhanced flexibility for interactive 3D segmentation based on statistical shape models \cite{Bernard:2015wx,bernard2016arXiv}).
\subsection{PCA Variants}
PCA is a non-convex problem that admits the efficient computation of the global optimum, e.g. by Singular Value Decomposition (SVD). However, the downside is that the incorporation of arbitrary (convex) regularisation terms is not possible due to the SVD-based solution. Therefore, incorporating regularisation terms into PCA is an active field of research and several variants have been presented: Graph-Laplacian PCA \cite{Jiang:2013wd} obtains factors with smoothly varying components according to a given graph. Robust PCA \cite{Candes:2011vf} formulates PCA as a convex low-rank matrix factorisation problem, where the nuclear norm is used as convex relaxation of the matrix rank.
 A combination of both Graph-Laplacian PCA and Robust PCA has been presented in \cite{Shahid:2015vb}. The Sparse PCA (SPCA) method \cite{Hein:2010wf,Zou:2006tl} obtains sparse factors. Structured Sparse PCA (SSPCA) \cite{Jenatton:2009tq} additionally imposes structure on the sparsity of the factors using group sparsity norms, such as the mixed $\ell_1/\ell_2$ norm. %

\subsection{Deformation Model Variants}
In \cite{Last:2011wk}, the flexibility of shape models has been increased by using PCA-based factors in combination with a \emph{per-vertex} weight vector, in contrast to a single weight vector that is used for all vertices. In \cite{Cootes:1995uz,Wang:2000iw}, it is shown that additional elasticity in the PCA-based model can be obtained by manipulation of the sample covariance matrix. Whilst both approaches increase the flexibility of the shape model, they result in global support factors.

In \cite{Sjostrand:2007hb}, SPCA is used to model the anatomical shape variation of the 2D outlines of the corpus callosum. 
In \cite{Uzumcu:2003gz}, 2D images of the cardiac ventricle were used to train an Active Appearance Model based on Independent Component Analysis (ICA) \cite{Hyvarinen:2001vj}.
Other applications of ICA for statistical shape models are presented in \cite{Suinesiaputra:2004jh,Zewail:2007ta}.
The Orthomax method, where the PCA basis is determined first and then rotated such that it has a ``simple'' structure, is used in \cite{Stegmann:2006ja}. 
The major drawback of SPCA, ICA and Orthomax is that the spatial relation between vertices is completely ignored. %

The Key Point Subspace Acceleration method based on Varimax, where a statistical subspace and key points are automatically identified from training data, is introduced in \cite{Meyer:2007wg}.
For mesh animation, in \cite{Tena:2011ts} the clusters of spatially close vertices are determined first by spectral clustering, and then PCA is applied for each vertex cluster, resulting in one sub-PCA model per cluster. This two-stage procedure has the problem, that, due to the independence of both stages, it is unclear whether the clustering is optimal with respect to the deformation model. 
Also, a blending procedure for the individual sub-PCA models is required. %
A similar approach of first manually segmenting body regions and then learning a PCA-based model has been presented in \cite{Yang:2014tj}.

The \emph{Sparse Localised Deformation Components} method (SPLOCS) obtains localised deformation modes from animated mesh sequences by using a matrix factorisation formulation with a weighted $\ell_1/\ell_2$ norm regulariser \cite{Neumann:2013gb}. Local support factors are obtained by explicitly modelling local support regions, which are in turn used to update the weights of the $\ell_1/\ell_2$ norm \emph{in each iteration}. This makes the non-convex optimisation problem even harder to solve and dismisses convergence guarantees.
With that, a suboptimal initialisation of the support regions, as presented in the work, affects the quality of the found solution. 

The \emph{compressed manifold modes} method \cite{Kovnatsky:2015tm,Neumann:2014uy} has the objective to obtain local support basis functions of the (discretised) Laplace-Beltrami operator of a \emph{single} input mesh.
In \cite{Kovnatsky:2014ti}, the authors obtain smooth functional correspondences between shapes that are spatially localised by using an $\ell_1$ norm regulariser in combination with the \emph{row} and \emph{column Dirichlet energy}.
The method proposed in \cite{Rustamov:2013tw} is able to localise \emph{shape differences} based on functional maps between two shapes.
Recently, the \emph{Shape-from-Operator} approach has been presented \cite{Boscaini:2015tl}, where shapes are reconstructed from more general intrinsic operators.

\subsection{Aims and Main Contributions}
The work presented in this paper has the objective of learning local support deformation factors from training data. 
The main application of the resulting shape model is recognition, segmentation and shape interpolation \cite{Bernard:2015wx,bernard2016arXiv}.
Whilst our work remedies several of the mentioned shortcomings of existing methods, it can also be seen as complementary to SPLOCS, which is more tailored towards artistic editing and mesh animation. 
The most significant difference to SPLOCS is that we aim at letting the training shapes \emph{automatically determine the location and size} of each local support region. 
This is achieved by formulating a matrix factorisation problem that incorporates regularisation terms which simultaneously account for \emph{sparsity} and \emph{smoothness} of the factors, where a graph-based smoothness regulariser accounts for smoothly varying neighbour vertices.
In contrast to SPLOCS or sub-PCA, this results in an implicit clustering that is part of the optimisation and does \emph{not require an initialisation} of local support regions, which in turn simplifies the optimisation procedure.
Moreover, by integrating a \emph{smoothness prior} into our framework we can \emph{handle small training datasets}, even if the desired number of factors exceeds the number of training shapes. 
Our optimisation problem is formulated in terms of the Structured Low-Rank Matrix Factorisation framework \cite{Haeffele:2014wj}, which has \emph{appealing theoretical properties}.

\section{Methods}
First, we introduce our notation and linear shape deformation models.
Then, we state the objective and its formulation as optimisation problem, followed by the theoretical motivation. Finally, the block coordinate descent algorithm and the factor splitting method are presented.

\subsection{Notation}
$\matI_p$ denotes the $p \times p$ identity matrix, $\onevec_p$ the $p$-dimensional vector containing ones, $\zerovec_{p\times q}$ the $p \times q$ zero matrix, and $\mathbb{S}^+_{p}$ the set of $p \times p$ positive semi-definite matrices. 
Let $\matA \in \R^{p \times q}$. We use the notation $\matA_{\mathcal{A},\mathcal{B}}$ to denote the submatrix of $\matA$ with the rows indexed by the (ordered) index set $\mathcal{A}$ and columns indexed by the (ordered) index set $\mathcal{B}$. The colon denotes the full (ordered) index set, e.g. $\matA_{\mathcal{A},:}$ is the matrix containing all rows of $\matA$ indexed by $\mathcal{A}$.
 For brevity, we write $\matA_r$ to denote the $p$-dimensional vector formed by the $r$-th column of $\matA$. 
The operator $\VEC(\matA)$ creates a $pq$-dimensional column vector from $\matA$ by stacking its columns, and $\otimes$ denotes the Kronecker product.
 
\subsection{Linear Shape Deformation Models}\label{lindefmod}
Let $\matX_k {\in} \R^{N \times 3}$ be the matrix representation of a shape comprising $N$ points (or vertices) in $3$ dimensions, and let ${\{\matX_k : 1 \leq k \leq K\}}$ be the set of $K$ training shapes. We assume that the rows in each $\matX_k$ correspond to homologous points. Using the vectorisation $\vecx_k {=} \VEC(\matX_k) {\in} \R^{3N}$, all $\vecx_k$ are arranged 
in the matrix ${\matX {=} [\vecx_1, \ldots, \vecx_K] {\in} \R^{3N{\times}K}}$. 
We assume that all shapes have the same pose, are centred at the mean shape $\bar{\matX}$, i.e.~$\sum_k \matX_k {=} \zerovec_{N {\times} 3}$, and that the standard deviation of $\VEC(\matX)$ is one.

Pairwise relations between vertices are modelled by a weighted undirected graph $\mathcal{G} {=} (\mathcal{V}, \mathcal{E}, \vecomega)$ that is shared by all shapes.  
The node set $\mathcal{V} {=}\{1,\ldots,N\}$ enumerates all $N$ vertices, the edge set ${\mathcal{E} {\subseteq} \{1,\ldots, N\}^2}$ represents the connectivity of the vertices, and $\vecomega {\in} \R_+^{\vert \mathcal{E} \vert}$ is the weight vector. The (scalar) weight $\vecomega_e$ of edge $e {=} (i,j) \in \mathcal{E}$ denotes the affinity between vertex $i$ and $j$, where ``close'' vertices have high value $\vecomega_e$.
We assume there are no self-edges, i.e.~$(i,i) {\notin} \mathcal{E}$. 
The graph can either encode pairwise \emph{spatial} connectivity, or affinities that are not of spatial nature (e.g. symmetries, or prior anatomical knowledge in medical applications).
For the standard PCA-based method \cite{Cootes:1992uw}, the modes of variation in the $M$ columns of the matrix $\matPhi {\in} \R^{3N \times M}$ are defined
as the $M$ most dominant eigenvectors of the sample covariance matrix $\matC {=} \frac{1}{K-1}\matX \matX^T$. However, we consider
the generalisation where $\matPhi$ contains general $3N$-dimensional vectors, the \emph{factors}, in its $M$ columns. In both cases, the (linear) deformation model (modulo the mean shape)
is given by
\begin{align}\label{pdmvec}
  \vecy(\vecalpha) = \matPhi \vecalpha \,,
\end{align}
with weight vector $\vecalpha \in \R^{M}$. Due to vectorisation, the rows with indices $\{1,\ldots,N\}, \{N{+}1,\ldots, 2N\}$ and $\{2N{+}1,\ldots,3N\}$ of $\vecy$ (or $\matPhi$), correspond to the $x$, $y$ and $z$ components of the $N$ vertices of the shape, respectively.

\subsection{Objective and Optimisation Problem}
The objective is to find ${\matPhi = [\matPhi_1, \ldots, \matPhi_M]}$ and ${\matAlpha = [\vecalpha_1, \ldots, \vecalpha_K] \in \R^{M \times K}}$ for a given $M < 3N$ such that, according to eq.~\eqref{pdmvec}, we can write 
\begin{align}
  \matX \approx \matPhi \matAlpha \,,
\end{align}
where the factors $\matPhi_m$ have \emph{local support}. Local support means that $\matPhi_m$ is sparse and that all \emph{active} vertices, i.e.~vertices that correspond to the non-zero elements of $\matPhi_m$, are connected by (sequences of) edges in the graph $\mathcal{G}$.

Now we state our problem formally as an optimisation problem. The theoretical motivation thereof is based on \cite{Bach:2008wn,Burer:2005jb,Haeffele:2014wj} and is recapitulated in section \ref{theomot}, where it will also become clear that our chosen regularisation term is related to the \emph{Projective Tensor Norm} \cite{Bach:2008wn,Haeffele:2014wj}. 

A general matrix factorisation problem with squared Frobenius norm loss is given by
\begin{align}
  \min_{\matPhi, \matA} \| \matX - \matPhi\matA \|_F^2 + \Omega(\matPhi,\matA)\,,
\end{align}
where the regulariser $\Omega$ imposes certain properties upon $\matPhi$ and $\matA$. The optimisation is performed over \emph{some} compact set (which we assume implicitly from here on).
An obvious property of local support factors is sparsity. Moreover, it is desirable that neighbour vertices vary smoothly. Both properties together seem to be promising candidates to obtain local support factors, which we reflect in our regulariser.
Our optimisation problem is given by
\begin{align}\label{optProb}
    \min_{{\substack{\matPhi{\in}\R^{{3N{\times}M}}\\ \matA{\in}\R^{M{\times}K}}}}\! \|\matX{-}\matPhi\matA\|_F^2 + 
    {\lambda}\sum_{m=1}^M \!\|\matPhi_m\|_\Phi \|(\matA_{m,:})^T\|_A \,, %
\end{align}
where $\|\cdot\|_\Phi$ and $\|\cdot\|_A$ denote vector norms. %
For $\vecz' \in \R^K, \vecz \in \R^{3N}$, we define
\begin{align}
  & \|\vecz'\|_A = %
  \lambda_{A} \| \vecz' \|_2 \,, \text{ and } \label{normA} \\
  &\|\vecz\|_\Phi = \lambda_{_1}\| \vecz \|_1 {+} 
   \lambda_{_2} \| \vecz \|_2  
   + \lambda_{\infty} \| \vecz \|_{1,\infty}^\mathcal{H} + 
    \lambda_{\mathcal{G}} \| \matE \vecz \|_2  \,.  \label{normPhi}
\end{align}
Both $\ell_2$ norm terms will be discussed in section \ref{theomot}. The $\ell_1$ norm is used to obtain sparsity in the factors. 
The (mixed) $\ell_1/\ell_\infty$ norm is defined by
\begin{align}\label{l1linf}
    \| \vecz \|_{1,\infty}^\mathcal{H} =  \sum_{g \in \mathcal{H}}\| \vecz_g \|_{\infty} \,,
\end{align}
where $\vecz_g$ denotes a subvector of $\vecz$ indexed by $g \in \mathcal{H}$. By using the collection ${\mathcal{H} = \{ \{i, i+N, i+2N\} : 1 \leq i \leq N\}}$, a grouping of the x, y and z components per vertex is achieved, i.e.~within a group $g$ only the component with largest magnitude is penalised and no extra cost is to be paid for the other components being non-zero.

The last term in eq.~\eqref{normPhi}, the graph-based $\ell_2$ (semi-)norm, imposes smoothness upon each factor, such that neighbour elements according to the graph $\mathcal{G}$ vary smoothly. 
Based on the incidence matrix of $\mathcal{G}$, we choose $\matE$ such that
  \begin{align}\label{matE}
      \| \matE \vecz \|_2 {=} \sqrt{\sum_{d \in \{0,N,2N\}} \sum_{(i,j) = e_p \in \mathcal{E}} \vecomega_{e_p} ( \vecz_{d+i} - \vecz_{d+j})^2}\,.
  \end{align}
  As such, $\matE$ is a discrete (weighted) gradient operator and $\|\matE \cdot\|_2^2$ corresponds to Graph-Laplacian regularisation \cite{Jiang:2013wd}. $\matE$ is specified in the supplementary material.

  The structure of our problem formulation in eqs.~\eqref{optProb}, \eqref{normA}, \eqref{normPhi} allows for various degrees-of-freedom in the form of the parameters. They allow to weigh the data term versus the regulariser ($\lambda$), control the rank of the solution ($\lambda_A$ and $\lambda_2$ together, cf.~last paragraph in section \ref{theomot}), control the sparsity ($\lambda_1$), control the amount of grouping of the x, y and z components ($\lambda_\infty$) and control the smoothness $\lambda_{\mathcal{G}}$. The number of factors $M$ has an impact on the size of the support regions (for small $M$ the regions tend to be larger, whereas for large $M$ they tend to be smaller).

\subsection{Theoretical Motivation}\label{theomot}
For a matrix $\matX \in \R^{3N \times K}$ and vector norms $\|\cdot\|_{\Phi}$ and $\|\cdot\|_{ A}$, let us define the function
\begin{align}
   \psi^M(\matX) = \label{ptn}
 \min_{\substack{\{(\matPhi \in \R^{3N{\times}M}, \\\matA\in\R^{M{\times}K}): \\ \matPhi \matA = \matX\}}} ~~\sum_{m=1}^M \|\matPhi_m\|_{\Phi} \| (\matA_{m,:})^T\|_{ A} \,.
\end{align}
The function $\psi(\cdot) {=} \lim_{M \rightarrow \infty} \psi^M(\cdot)$
defines a norm known as \emph{Projective Tensor Norm} or \emph{Decomposition Norm} \cite{Bach:2008wn,Haeffele:2014wj}.
\begin{lemma}
  For any $\epsilon > 0$ there exists an $M(\epsilon) \in \N$ such that $\|\psi(\matX) - \psi^{M(\epsilon)}(\matX)\| < \epsilon$.
\end{lemma}
\begin{proof}
For $\psi(\matX)$ there are sequences $\{\matPhi_i\}_{i=1}^{\infty}$ and $\{\matA^T_i\}_{i = 1}^{\infty}$ such that $\psi(\matX) = \sum_{i = 1}^{\infty}\|\matPhi_i\|_{\Phi}\|\matA^T_i\|_{A}$. 
Let $l_m = \sum_{i = 1}^{m}\|\matPhi_i\|_{\Phi} \|\matA^T_i\|_{A}$. The sequence
$l_m$ is monotone, bounded from above and convergent. Let 
${l_{\infty} = \psi(\matX)}$ denote its limit. Since the sequence is convergent, there is 
$M(\epsilon)$ such that $\|l_{\infty} - l_j\| < \epsilon$ for ${j \geq M(\epsilon)}$.
\end{proof}
\noindent We now proceed by introducing the optimisation problem
\begin{align}\label{zprob}
    \min_\matZ \| \matX - \matZ \|_F^2 + \lambda \psi^M(\matZ) \,.
  \end{align}
Next, we establish a connection between problem \eqref{zprob} and our problem \eqref{optProb}. 
First, we assume that we are given a solution pair $(\matPhi,\matA)$ minimising problem \eqref{optProb}. By defining $\matZ = \matPhi\matA$, $\matZ$ is a solution to problem \eqref{zprob}. 
Secondly, assume we are given a solution $\matZ$ minimising problem \eqref{zprob}. To find a solution solution pair $(\matPhi,\matA)$ minimising problem \eqref{optProb}, one needs to compute the $(\matPhi,\matA)$ that achieves the minimum of the right-hand side of \eqref{ptn} for a given $\matZ$.

This shows that given a solution to one of the problems, one can infer a solution to the other problem. Next we reformulate problem \eqref{zprob}.
Following \cite{Haeffele:2014wj}, we define the matrices $\matQ  \in \R^{3N+K \times M},~\matY \in \R^{3N+K \times 3N+K}$ as
 \begin{align}
    & \matQ = \begin{bmatrix}
      \matPhi \\ \matA^T 
    \end{bmatrix}\,, \quad
    \matY = \matQ\matQ^T = \begin{pmatrix}
       \matPhi\matPhi^T & \matPhi \matA \\
       \matA^T \matPhi^T & \matA^T \matA
     \end{pmatrix}\,,
    \intertext{and the function $F: \mathbb{S}_{3N+K}^+ \rightarrow \R$ as}
    & F(\matY) = F(\matQ\matQ^T) = \| \matX - \matPhi\matA \|_F^2 + \lambda \psi^M(\matPhi\matA)\,.
  \end{align}
  \begin{flalign}\label{optProbCvx}
\text{Let} && \matY^* = \argmin_{\matY \in \mathbb{S}_{3N+K}^+} F(\matY)\,. \quad\qquad\qquad
\end{flalign}
For a given $\matY^*$, problem \eqref{zprob} is minimised by the upper-right block matrix of $\matY^*$.
The difference between \eqref{zprob} and \eqref{optProbCvx} is that the latter is over the set of positive semi-definite matrices, which, at first sight, does not present any gain.
However, under certain conditions, the global solution for $\matQ$, rather than the product $\matY = \matQ\matQ^T$, can be obtained directly \cite{Burer:2005jb}. %
In \cite{Bach:2008wn} it is shown that
if $\matQ$ is a \emph{rank deficient} local minimum of $F(\matQ\matQ^T)$, then it is also a global minimum. Whilst these results only hold for twice differentiable functions $F$, Haeffele et al.\ have presented analogous results for the case of $F$ being a sum of a twice-differentiable term and a non-differentiable term \cite{Haeffele:2014wj}, such as ours above.

As such, \emph{any} (rank deficient) local optimum of problem \eqref{optProb} is also a global optimum.
If in $\psi(\cdot)$, both $\|\cdot\|_{\Phi}$ and $\|\cdot\|_{ A}$ are the $\ell_2$ norm, $\psi(\cdot)$ is equivalent to the nuclear norm,
commonly used as convex relaxation of the matrix rank \cite{Haeffele:2014wj,Recht:2010ht}.
In order to steer the solution towards being rank deficient, we include $\ell_2$ norm terms in $\|\cdot\|_\Phi$ and $\|\cdot\|_A$ (see \eqref{normA} and \eqref{normPhi}). With that, part of the regularisation term in \eqref{optProb} is the nuclear norm that accounts for low-rank solutions. 
 
\subsection{Block Coordinate Descent}
A solution to problem \eqref{optProb} is found by block coordinate descent (BCD) \cite{Xu:2013vk} (algorithm~\ref{bcdCode}). It employs alternating proximal steps, which can be seen as generalisation of gradient steps for non-differentiable functions.
%
%
%
\begin{algorithm}\label{bcdCode}
\scriptsize
\SetKwInput{Input}{Input}
\SetKwInput{Output}{Output}
\DontPrintSemicolon
 \Repeat{convergence}{
   \tcp{update $\matPhi$}
   $\matPhi' \leftarrow \matPhi - \epsilon_\Phi \nabla_\Phi \| \matX - \matPhi \matA \|_F^2$\tcp*{gradient step (loss)}
   \For( \tcp*[h]{proximal step $\matPhi$ (penalty)}){$m=1,\ldots,M$}{
      $\matPhi_m  \leftarrow \prox_{\lambda \|\cdot\|_\Phi \|(\matA_{m,:})^T\|_A}(\matPhi'_m )$
   }
   \BlankLine
   \tcp{update $\matA$}
   $\matA' \leftarrow \matA- \epsilon_A \nabla_A\| \matX - \matPhi \matA \|_F^2$\tcp*{gradient step (loss)}
   
   \For( \tcp*[h]{proximal step $\matA$ (penalty)}){$m=1,\ldots,M$}{%
      $\matA_{m,:}  \leftarrow \prox_{\lambda \|\matPhi_m\|_\Phi \|\cdot\|_A}((\matA^{'}_{m,:})^T)^T$
   }
   }
 \caption{Simplified view of block coordinate descent. For details see \cite{Haeffele:2014wj,Xu:2013vk}.}
\end{algorithm}
 %
%
Since computing the proximal mapping is repeated in each iteration, an efficient computation is essential.
The proximal mapping of $\|\cdot\|_A$ in eq.~\eqref{normA} has a closed-form solution by \emph{block soft thresholding} \cite{Parikh:2013vb}.
The proximal mapping of $\|\cdot\|_\Phi$ in eq.~\eqref{normPhi} is solved by dual forward-backward splitting \cite{Combettes:2010uu,Combettes:2009wd} (see supplementary material).
The benefit of BCD is that it scales well to large problems (cf. complexity analysis in the supplementary material). However, a downside is that by using the alternating updates one has only guaranteed convergence to a Nash equilibrium point \cite{Haeffele:2014wj,Xu:2013vk}. %

\subsection{Factor Splitting}
Whilst solving problem \eqref{optProb} leads to smooth and sparse factors, there is no guarantee that the factors have only a \emph{single} local support region. %
In fact, as motivated in section \ref{theomot}, the solution of eq.~\eqref{optProb} is steered towards being low-rank. However, the price to pay for a low-rank solution is that capturing multiple support regions in a \emph{single} factor is preferred over having each support region in an individual factor (e.g. for $M=2$ and any $\veca \neq \zerovec, ~\vecb \neq \zerovec$, the matrix $\matPhi = [\matPhi_1 \matPhi_2]$ has a lower rank when $\matPhi_1 = [\veca^T ~\vecb^T]^T$ and $\matPhi_2 = \zerovec$, compared to $\matPhi_1 = [\veca^T ~\zerovec^T]^T$ and  $\matPhi_2 = [\zerovec^T~ \vecb^T]^T$).

A simple yet effective way to deal with this issue is to split factors with multiple disconnected support regions into multiple (new) factors (see supplementary material). %
Since this is performed \emph{after} the optimisation problem has been solved, it is preferable over pre- or intra-processing procedures \cite{Tena:2011ts,Neumann:2013gb} since the optimisation remains unaffected and the data term in eq.~\eqref{optProb} does not change due to the splitting. %

\setlength{\abovedisplayskip}{2pt}
\setlength{\belowdisplayskip}{2pt}
\section{Experimental Results}\label{expres}
We compared the proposed method with \emph{PCA} \cite{Cootes:1992uw}, \emph{kernel PCA} (kPCA, cf.~\ref{kernelisation}), \emph{Varimax} \cite{Harman:1976vh}, \emph{ICA} \cite{Hyvarinen:2001vj}, \emph{SPCA} \cite{Jenatton:2009tq}, \emph{SSPCA} \cite{Jenatton:2009tq}, and \emph{SPLOCS} \cite{Neumann:2013gb} on two real datasets, brain structures and human body shapes. Only our method and the SPLOCS method explicitly aim to obtain local support factors, whereas the SPCA and SSPCA methods obtain sparse factors (for the latter the $\ell_1/\ell_2$ norm with groups defined by $\mathcal{H}$, cf.~eq.~\eqref{l1linf}, is used). The methods kPCA, SPCA, SSPCA, SPLOCS and ours require to set various parameters, which we address by random sampling (see supplementary material).

For all evaluated methods a factorisation $\matPhi \matA$ is obtained. W.l.o.g. we normalise the rows of $\matA$ to have standard deviation one (since $\matPhi \matA = (\frac{1}{s}\matPhi) (s\matA)$ for $s\neq0$). Then, the factors in $\matPhi$ are ordered descendingly according to their $\ell_2$ norms. 

In our method, the number of factors may be changed due to factor splitting, thus, in order to allow a fair comparison, we only retain the first $M$ factors.
Initially, the columns of $\matPhi$ are chosen to $M$ (unique) columns selected randomly from $\matI_{3N}$. This is in accordance with \cite{Haeffele:2014wj}, where empirically good results are obtained using trivial initialisations. Convergence plots for different initialisations can be found in the supplementary material.

\subsection{Quantitative Measures}\label{quanmeas}
For $\vecx = \VEC(\matX)$ and $\tilde\vecx = \VEC(\tilde\matX)$, the \emph{average error}
\begin{align}
  e_{\text{avg}}(\vecx,\tilde\vecx) = \frac{1}{N} \sum_{i=1}^N \| \matX_{i,:} {-} \tilde\matX_{i,:}\|_2 
  \intertext{and the \emph{maximum error}}
   e_{\max}(\vecx,\tilde\vecx)  =\max_{i=1,\ldots,N} \| \matX_{i,:} - \tilde\matX_{i,:}\|_2 
\end{align}
measure the agreement between shape $\matX$ and shape $\tilde\matX$. %

The \emph{reconstruction error} for shape $\vecx_k$ is measured by solving the overdetermined system $\matPhi \vecalpha_k = \vecx_k$ for $\vecalpha_k$ in the least-squares sense, and then computing  $e_{\text{avg}}(\vecx_k, \matPhi \vecalpha_k)$ and $e_{\max}(\vecx_k, \matPhi \vecalpha_k)$, respectively.

To measure the \emph{specificity error}, $n_s$ shape samples are drawn randomly ($n_s {=} 1000$ for the brain shapes and $n_s {=} 100$ for the body shapes). For each drawn shape, the average and maximum errors between the closest shape in the training set are denoted by $s_{\text{avg}}$ and $s_{\max}$, respectively. 
For simplicity, we assumed that the parameter vector $\vecalpha$ follows a zero mean Gaussian distribution, where the covariance matrix $\matC_{\vecalpha}$ is estimated from the parameter vectors $\vecalpha_k$ of the $K$ training shapes. With that, a random shape sample $\vecx_r$ is generated by drawing a sample of the vector $\vecalpha_r$ from its distribution and setting $\vecx_r = \matPhi \vecalpha_r$.
The specificity can be interpreted as a score for assessing how realistic synthesized shapes are on a coarse level of detail (the contribution of errors on fine details to the specificity score is marginal due to the dominance of the errors on coarse scales). %

For evaluating the \emph{generalisation error}, a collection of index sets $\mathcal{I} \subset 2^{\{1,\ldots,K\}}$ is used, where each set $\mathcal{J} \in \mathcal{I}$ denotes the set of indices of the \emph{test shapes} for one run and $\vert \mathcal{I} \vert$ is the number of runs. We used five-fold cross-validation, i.e.~$\vert \mathcal{I} \vert = 5$ and each set $\mathcal{J}$ contains $\operatorname{round}(\frac{K}{5})$ random integers from $\{1, \ldots, K\}$.
In each run, the deformation factors $\matPhi^{\mathcal{\bar J}}$ are computed using only the shapes with indices ${\mathcal{\bar J}} =  \{1,\ldots,K\} \setminus \mathcal{J}$. For all test shapes $\vecx_j$, where $j \in \mathcal{J}$, the parameter vector $\vecalpha_j$ is determined by solving $\matPhi^{\mathcal{\bar J}} \vecalpha_j = \vecx_j$ in the least-squares sense. Eventually, the average reconstruction error $e_{\text{avg}}(\vecx_j,\matPhi^{\mathcal{\bar J}} \vecalpha_j)$ and the maximum reconstruction error $e_{\max}(\vecx_j,\matPhi^{\mathcal{\bar J}} \vecalpha_j)$ are computed for each test shape, which we denote as $g_{\text{avg}}$ and $g_{\max}$, respectively.
Moreover, the \emph{sparse reconstruction errors} $g_{\text{avg}}^{0.05}$ and $g_{\max}^{0.05}$ are computed in a similar manner, with the difference that $\vecalpha_j$ is now determined by using only $5\%$ of the rows (selected randomly) of $\matPhi^{\mathcal{\bar J}}$ and $\vecx_j$, denoted by $\tilde\matPhi^{\mathcal{\bar J}}$ and $\tilde\vecx_j$. For that, we minimise $\| \tilde\matPhi^{\mathcal{\bar J}} \vecalpha_j - \tilde\vecx_j\|_2^2 + \| \Gamma \vecalpha_j \|_2^2$ with respect to $\vecalpha_j$, which is a least-squares problem with Tikhonov regularisation, where $\Gamma$ is obtained by Cholesky factorisation of $\matC_{\vecalpha} = \Gamma^T \Gamma$.
The purpose of this measure is to evaluate how well a deformation model interpolates an unseen shape from a small subset of its points, which is relevant for shape model-based surface reconstruction \cite{Bernard:2015wx,bernard2016arXiv}.

\subsection{Brain Structures}
The first evaluated dataset comprises $8$ brain structure meshes of ${K{=}17}$ subjects \cite{Bernard:2016tv}. 
All $8$ meshes together have a total number of ${N{=}1792}$ vertices that are in correspondence among all subjects. Moreover, all meshes have the same topology, i.e.~seen as graphs they are isomorphic. A single deformation model is used to model the deformation of all $8$ meshes in order to capture the interrelation between the brain structures. %
We fix the number of desired factors to $M{=}96$ to account for a sufficient amount of local details in the factors.
Whilst the meshes are smooth and comparably simple in shape (cf.~Fig.~\ref{brainDefFactors}), a particular challenge is that the training dataset comprises only $K{=}17$ shapes. 

\subsubsection{Second-order Terms}
Based on anatomical knowledge, we use the brain structure interrelation graph $\mathcal{G}_{\text{bs}} = (\mathcal{V}_{\text{bs}}, \mathcal{E}_{\text{bs}})$ shown in Fig.~\ref{bsgraph}, where an edge between two structures denotes that a deformation of one structure may have a direct effect on the deformation of the other structure.
\begin{figure}[h!]
      \vspace{-4mm}
     \centerline{\includegraphics[scale=0.73]{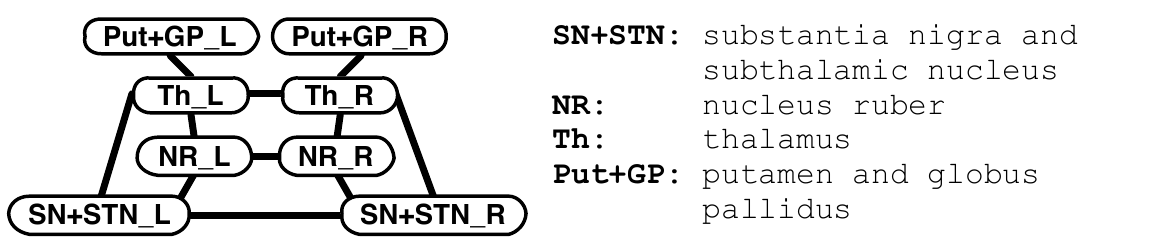}}
     \caption{Brain structure interrelation graph.}%
     \label{bsgraph}\vspace{-3mm}
\end{figure}
Using $\mathcal{G}_{\text{bs}}$, we now build a distance matrix that is then used in the SPLOCS method and our method. 
For $\mathfrak{o} \in \mathcal{V}_{\text{bs}}$, let $g_\mathfrak{o} \subset \{1,\ldots,N\}$ denote the (ordered) indices of the vertices of brain structure $\mathfrak{o}$. Let $\matD_{\text{euc}} \in \R^{N \times N}$ be the Euclidean distance matrix, where $(\matD_{\text{euc}})_{ij} = \| \bar{\matX}_{i,:} - \bar{\matX}_{j,:} \|_2$ is the Euclidean distance between vertex $i$ and $j$ of the mean shape $\bar{\matX}$. Moreover, let $\matD_{\text{geo}} \in \R^{N \times N}$ be the geodesic graph distance matrix of the mean shape $\bar{\matX}$ using the graph induced by the (shared) mesh topology. By $\matD_{\text{euc}}^{\mathfrak{o}} \in \R^{\vert g_\mathfrak{o} \vert \times \vert g_\mathfrak{o} \vert}$ and $\matD_{\text{geo}}^{\mathfrak{o}} \in \R^{\vert g_\mathfrak{o} \vert \times \vert g_\mathfrak{o} \vert}$ we denote the Euclidean distance matrix and the geodesic distance matrix of brain structure $\mathfrak{o}$, which are submatrices of $\matD_{\text{euc}}$ and $\matD_{\text{geo}}$, respectively. Let $\bar{d}^{\mathfrak{o}}$ denote the average vertex distance between neighbour vertices of brain structure $\mathfrak{o}$. We define the normalised geodesic graph distance matrix of brain structure $\mathfrak{o}$ by $\tilde\matD_{\text{geo}}^{\mathfrak{o}} = \frac{1}{\bar{d}^{\mathfrak{o}}}\matD_{\text{geo}}^{\mathfrak{o}}$ and the matrix $\tilde\matD_{\text{geo}} \in \R^{N \times N}$ is composed by the individual blocks $\tilde\matD_{\text{geo}}^{\mathfrak{o}}$. 

The normalised distance matrix between structure $\mathfrak{o}$ and $\mathfrak{\tilde o}$ is given by $\tilde\matD_{\text{bs}}^{\mathfrak{o},\mathfrak{\tilde o}} = \frac{2}{\bar{d}^{\mathfrak{o}}+ \bar{d}^{\mathfrak{\tilde o}}}[(\matD_{\text{euc}})_{g_{\mathfrak{o}}, g_{\mathfrak{\tilde o}}} {-} \onevec_{\vert g_{\mathfrak{o}} \vert} \onevec_{\vert g_{\mathfrak{\tilde o}}\vert}^T d_{\min}^{\mathfrak{o},\mathfrak{\tilde o}}] \in \R^{\vert g_{\mathfrak{o}} \vert \times \vert g_{\mathfrak{\tilde o}}\vert}$, where $d_{\min}^{\mathfrak{o},\mathfrak{\tilde o}}$ is the smallest element of $(\matD_{\text{euc}})_{g_{\mathfrak{o}}, g_{\mathfrak{\tilde o}}}$. The (symmetric) distance matrix $\tilde\matD_{\text{bs}} \in \R^{N \times N}$ between all structures is constructed by
\begin{align}
  (\tilde\matD_{\text{bs}})_{g_{\mathfrak{o}}, g_{\mathfrak{\tilde o}}} = \begin{cases}
    \tilde\matD_{\text{bs}}^{\mathfrak{o},\mathfrak{\tilde o}} & \text{ if } \quad (\mathfrak{o},\mathfrak{\tilde o}) \in \mathcal{E}_{\text{bs}} \\
    \zerovec_{\vert g_{\mathfrak{o}} \vert \times \vert g_{\mathfrak{\tilde o}}\vert} & \text{ else } 
  \end{cases} \,.
\end{align}
For the SPLOCS method we used the distance matrix $\matD = \alpha_{D} \tilde \matD_{\text{geo}} + (1-\alpha_{D}) \tilde \matD_{\text{bs}}$. For our method, we construct the graph $\mathcal{G} = (\mathcal{V}, \mathcal{E}, \vecomega)$ (cf.~section \ref{lindefmod}) by having an edge $e=(i,j)$ in $\mathcal{E}$ for each $\vecomega_e = \alpha_{D}\exp( - (\tilde \matD_{\text{geo}})_{i,j}^2) + (1-\alpha_D) \exp(- (\tilde\matD_{\text{bs}})_{ij}^2)$ that is larger than the threshold $\theta = 0.1$. In both cases we set $\alpha_D = 0.5$.

\begin{figure*}[t!]%
     \centerline{\includegraphics[scale=0.247]{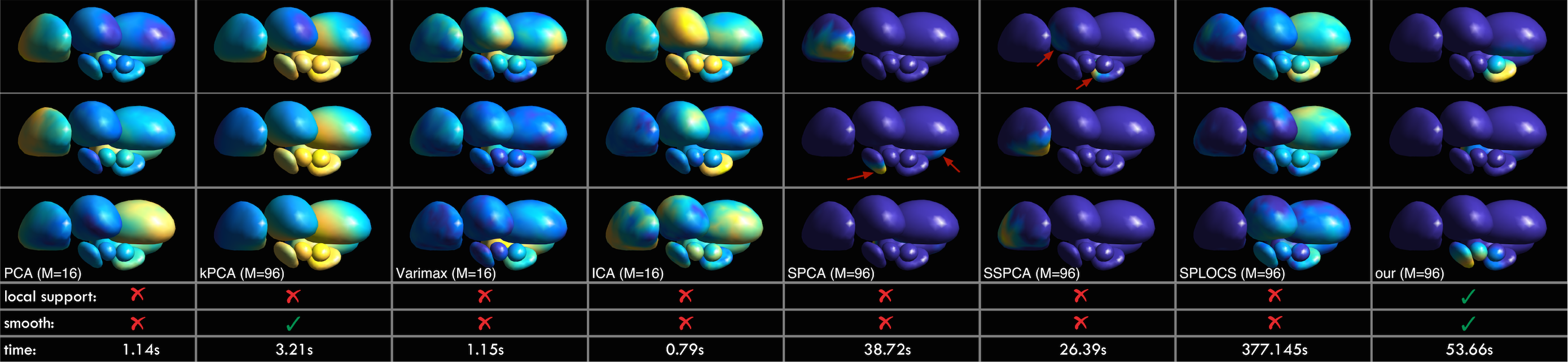}}
      \vspace{-2mm}
     \caption{The
     colour-coded magnitude (blue corresponds to zero, yellow to the maximum deformation in each plot) for the three deformation factors with largest $\ell_2$ norm is shown in the rows. The factors obtained by SPCA and SSPCA are sparse but not spatially localised (see red arrows). Our method is the only one that obtains local support factors.}
     \label{brainDefFactors}
\end{figure*}
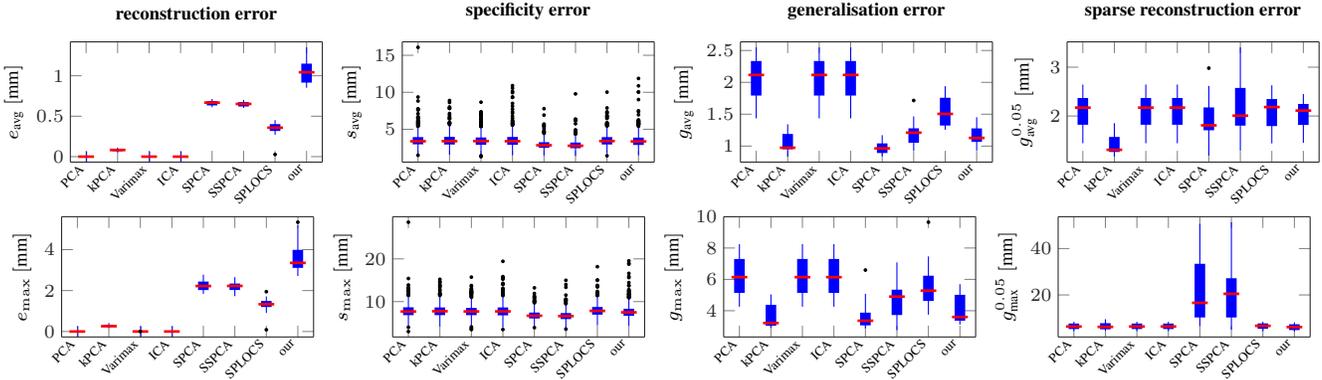
\begin{figure*}%
     \centerline{
        \subfigure{%
%
%
\begin{tikzpicture}

\begin{axis}[%
width=3.35cm,
height=1.6cm,
unbounded coords=jump,
scale only axis,
separate axis lines,
every outer x axis line/.append style={white!15!black},
every x tick label/.append style={font=\color{white!15!black}},
xmin=0.5,
xmax=8.5,
xtick={1,2,3,4,5,6,7,8},
xticklabels={{PCA},{kPCA},{Varimax},{ICA},{SPCA},{SSPCA},{SPLOCS},{our}},
every outer y axis line/.append style={white!15!black},
every y tick label/.append style={font=\color{white!15!black}},
ymin=-0.0675869005407082,
ymax=1.41932491135501,
ylabel={$e_{\text{avg}}~[\text{mm}]$},
title style={font=\bfseries},
title={reconstruction error},
every x tick label/.append style={rotate=45, anchor=east, align=left, font=\tiny}
]
\addplot [color=blue,solid,forget plot]
  table[row sep=crcr]{%
1	6.14272347202661e-15\\
1	2.35425519750532e-14\\
};
\addplot [color=blue,solid,forget plot]
  table[row sep=crcr]{%
2	0.0519277153967054\\
2	0.112679469225825\\
};
\addplot [color=blue,solid,forget plot]
  table[row sep=crcr]{%
3	7.19159331768924e-15\\
3	2.76673643384708e-14\\
};
\addplot [color=blue,solid,forget plot]
  table[row sep=crcr]{%
4	6.74821385861984e-15\\
4	3.77432791455962e-14\\
};
\addplot [color=blue,solid,forget plot]
  table[row sep=crcr]{%
5	0.611411753257908\\
5	0.713262028216746\\
};
\addplot [color=blue,solid,forget plot]
  table[row sep=crcr]{%
6	0.601195688481272\\
6	0.701796638434952\\
};
\addplot [color=blue,solid,forget plot]
  table[row sep=crcr]{%
7	0.269711512458102\\
7	0.45189717050027\\
};
\addplot [color=blue,solid,forget plot]
  table[row sep=crcr]{%
8	0.851709741389065\\
8	1.35173801081429\\
};
\addplot [color=blue,solid,line width=4.0pt,forget plot]
  table[row sep=crcr]{%
1	8.02999507596591e-15\\
1	1.49216161763244e-14\\
};
\addplot [color=blue,solid,line width=4.0pt,forget plot]
  table[row sep=crcr]{%
2	0.0676192888468827\\
2	0.0868680032319096\\
};
\addplot [color=blue,solid,line width=4.0pt,forget plot]
  table[row sep=crcr]{%
3	1.19447978661095e-14\\
3	1.96397178248008e-14\\
};
\addplot [color=blue,solid,line width=4.0pt,forget plot]
  table[row sep=crcr]{%
4	1.12435768827984e-14\\
4	2.80239872810181e-14\\
};
\addplot [color=blue,solid,line width=4.0pt,forget plot]
  table[row sep=crcr]{%
5	0.63328298370827\\
5	0.687025805241334\\
};
\addplot [color=blue,solid,line width=4.0pt,forget plot]
  table[row sep=crcr]{%
6	0.620052545144967\\
6	0.674394392001618\\
};
\addplot [color=blue,solid,line width=4.0pt,forget plot]
  table[row sep=crcr]{%
7	0.317368969627774\\
7	0.402250865514412\\
};
\addplot [color=blue,solid,line width=4.0pt,forget plot]
  table[row sep=crcr]{%
8	0.915507042026091\\
8	1.14995016812364\\
};
\addplot [color=red,solid,line width=1.0pt,forget plot]
  table[row sep=crcr]{%
0.75	1.00463634361022e-14\\
1.25	1.00463634361022e-14\\
};
\addplot [color=red,solid,line width=1.0pt,forget plot]
  table[row sep=crcr]{%
1.75	0.0821794814660452\\
2.25	0.0821794814660452\\
};
\addplot [color=red,solid,line width=1.0pt,forget plot]
  table[row sep=crcr]{%
2.75	1.32684446943313e-14\\
3.25	1.32684446943313e-14\\
};
\addplot [color=red,solid,line width=1.0pt,forget plot]
  table[row sep=crcr]{%
3.75	1.55600195745122e-14\\
4.25	1.55600195745122e-14\\
};
\addplot [color=red,solid,line width=1.0pt,forget plot]
  table[row sep=crcr]{%
4.75	0.667969141093421\\
5.25	0.667969141093421\\
};
\addplot [color=red,solid,line width=1.0pt,forget plot]
  table[row sep=crcr]{%
5.75	0.651721008227043\\
6.25	0.651721008227043\\
};
\addplot [color=red,solid,line width=1.0pt,forget plot]
  table[row sep=crcr]{%
6.75	0.357974070696189\\
7.25	0.357974070696189\\
};
\addplot [color=red,solid,line width=1.0pt,forget plot]
  table[row sep=crcr]{%
7.75	1.04331156709113\\
8.25	1.04331156709113\\
};
\addplot [color=black,mark size=0.5pt,only marks,mark=*,mark options={solid,fill=black,draw=black},forget plot]
  table[row sep=crcr]{%
nan	nan\\
};
\addplot [color=black,mark size=0.5pt,only marks,mark=*,mark options={solid,fill=black,draw=black},forget plot]
  table[row sep=crcr]{%
nan	nan\\
};
\addplot [color=black,mark size=0.5pt,only marks,mark=*,mark options={solid,fill=black,draw=black},forget plot]
  table[row sep=crcr]{%
nan	nan\\
};
\addplot [color=black,mark size=0.5pt,only marks,mark=*,mark options={solid,fill=black,draw=black},forget plot]
  table[row sep=crcr]{%
nan	nan\\
};
\addplot [color=black,mark size=0.5pt,only marks,mark=*,mark options={solid,fill=black,draw=black},forget plot]
  table[row sep=crcr]{%
nan	nan\\
};
\addplot [color=black,mark size=0.5pt,only marks,mark=*,mark options={solid,fill=black,draw=black},forget plot]
  table[row sep=crcr]{%
nan	nan\\
};
\addplot [color=black,mark size=0.5pt,only marks,mark=*,mark options={solid,fill=black,draw=black},forget plot]
  table[row sep=crcr]{%
7	0.0271989031186847\\
};
\addplot [color=black,mark size=0.5pt,only marks,mark=*,mark options={solid,fill=black,draw=black},forget plot]
  table[row sep=crcr]{%
nan	nan\\
};
\end{axis}
\end{tikzpicture}%
%
 } \hfil
        \subfigure{%
%
%
\begin{tikzpicture}

\begin{axis}[%
width=3.35cm,
height=1.6cm,
scale only axis,
separate axis lines,
every outer x axis line/.append style={white!15!black},
every x tick label/.append style={font=\color{white!15!black}},
xmin=0.5,
xmax=8.5,
xtick={1,2,3,4,5,6,7,8},
xticklabels={{PCA},{kPCA},{Varimax},{ICA},{SPCA},{SSPCA},{SPLOCS},{our}},
every outer y axis line/.append style={white!15!black},
every y tick label/.append style={font=\color{white!15!black}},
ymin=0.560253866421579,
ymax=16.8182087286496,
ylabel={$s_{\text{avg}}~[\text{mm}]$},
title style={font=\bfseries},
title={specificity error},
every x tick label/.append style={rotate=45, anchor=east, align=left, font=\tiny}
]
\addplot [color=blue,solid,forget plot]
  table[row sep=crcr]{%
1	1.61274909020978\\
1	5.34028829446794\\
};
\addplot [color=blue,solid,forget plot]
  table[row sep=crcr]{%
2	1.5558530450168\\
2	5.5505960703576\\
};
\addplot [color=blue,solid,forget plot]
  table[row sep=crcr]{%
3	1.48799259342314\\
3	5.36728234084417\\
};
\addplot [color=blue,solid,forget plot]
  table[row sep=crcr]{%
4	1.50253169319392\\
4	5.50549445176275\\
};
\addplot [color=blue,solid,forget plot]
  table[row sep=crcr]{%
5	1.44944890596315\\
5	4.3551959635788\\
};
\addplot [color=blue,solid,forget plot]
  table[row sep=crcr]{%
6	1.40238862374653\\
6	4.25986597697911\\
};
\addplot [color=blue,solid,forget plot]
  table[row sep=crcr]{%
7	1.75800944882958\\
7	5.40644305479832\\
};
\addplot [color=blue,solid,forget plot]
  table[row sep=crcr]{%
8	1.48485372018121\\
8	5.3653475361935\\
};
\addplot [color=blue,solid,line width=4.0pt,forget plot]
  table[row sep=crcr]{%
1	2.96624295245305\\
1	3.9424006199701\\
};
\addplot [color=blue,solid,line width=4.0pt,forget plot]
  table[row sep=crcr]{%
2	2.93602113126632\\
2	3.98185906296982\\
};
\addplot [color=blue,solid,line width=4.0pt,forget plot]
  table[row sep=crcr]{%
3	2.90776339720572\\
3	3.89433093316106\\
};
\addplot [color=blue,solid,line width=4.0pt,forget plot]
  table[row sep=crcr]{%
4	2.89991511055708\\
4	3.96045075049696\\
};
\addplot [color=blue,solid,line width=4.0pt,forget plot]
  table[row sep=crcr]{%
5	2.49761589611922\\
5	3.25041763961088\\
};
\addplot [color=blue,solid,line width=4.0pt,forget plot]
  table[row sep=crcr]{%
6	2.43993810013827\\
6	3.2082808318406\\
};
\addplot [color=blue,solid,line width=4.0pt,forget plot]
  table[row sep=crcr]{%
7	2.94384727577783\\
7	3.942720884633\\
};
\addplot [color=blue,solid,line width=4.0pt,forget plot]
  table[row sep=crcr]{%
8	2.87843243623095\\
8	3.87437951164573\\
};
\addplot [color=red,solid,line width=1.0pt,forget plot]
  table[row sep=crcr]{%
0.75	3.38197791320443\\
1.25	3.38197791320443\\
};
\addplot [color=red,solid,line width=1.0pt,forget plot]
  table[row sep=crcr]{%
1.75	3.39034834256747\\
2.25	3.39034834256747\\
};
\addplot [color=red,solid,line width=1.0pt,forget plot]
  table[row sep=crcr]{%
2.75	3.40307545658499\\
3.25	3.40307545658499\\
};
\addplot [color=red,solid,line width=1.0pt,forget plot]
  table[row sep=crcr]{%
3.75	3.38482311244629\\
4.25	3.38482311244629\\
};
\addplot [color=red,solid,line width=1.0pt,forget plot]
  table[row sep=crcr]{%
4.75	2.83791654701367\\
5.25	2.83791654701367\\
};
\addplot [color=red,solid,line width=1.0pt,forget plot]
  table[row sep=crcr]{%
5.75	2.76802726877916\\
6.25	2.76802726877916\\
};
\addplot [color=red,solid,line width=1.0pt,forget plot]
  table[row sep=crcr]{%
6.75	3.40625910896863\\
7.25	3.40625910896863\\
};
\addplot [color=red,solid,line width=1.0pt,forget plot]
  table[row sep=crcr]{%
7.75	3.36011589078863\\
8.25	3.36011589078863\\
};
\addplot [color=black,mark size=0.5pt,only marks,mark=*,mark options={solid,fill=black,draw=black},forget plot]
  table[row sep=crcr]{%
1	1.48300001079589\\
1	5.40733469682028\\
1	5.46596975950147\\
1	5.47620994289079\\
1	5.61909714205232\\
1	5.62468342809035\\
1	5.63770673730661\\
1	5.69691991397461\\
1	5.75652069351862\\
1	5.83867635776295\\
1	5.84374367861011\\
1	5.90667815288234\\
1	5.94463365374561\\
1	5.96414770534703\\
1	6.0378533966862\\
1	6.0396147160261\\
1	6.10471247477834\\
1	6.1667337093846\\
1	6.25550567628552\\
1	6.44759479752337\\
1	6.50674741982778\\
1	6.60610905889881\\
1	6.6106776743934\\
1	6.67480502389853\\
1	7.09965512550056\\
1	7.61201213708997\\
1	7.64739555362303\\
1	7.72969638772739\\
1	7.75473827741911\\
1	8.82079603049667\\
1	9.35839373312072\\
1	16.0792107803665\\
};
\addplot [color=black,mark size=0.5pt,only marks,mark=*,mark options={solid,fill=black,draw=black},forget plot]
  table[row sep=crcr]{%
2	5.63378688115871\\
2	5.74012423045986\\
2	5.80566471261527\\
2	5.84946807587703\\
2	5.85730229386448\\
2	5.87914687195224\\
2	5.90258234815107\\
2	6.40759163854216\\
2	6.45271612691983\\
2	6.51489720948685\\
2	6.69906498886557\\
2	6.92481792351411\\
2	7.12694397407046\\
2	7.83025764357823\\
2	7.83226970909869\\
2	7.84738995436054\\
2	8.01214018058269\\
2	8.62601342737523\\
2	8.88320994153796\\
};
\addplot [color=black,mark size=0.5pt,only marks,mark=*,mark options={solid,fill=black,draw=black},forget plot]
  table[row sep=crcr]{%
3	1.29925181470467\\
3	1.42018738427107\\
3	5.39648223848798\\
3	5.40421282868032\\
3	5.42606600252007\\
3	5.43359696117682\\
3	5.43762213957059\\
3	5.43868954514719\\
3	5.45556389718129\\
3	5.6468302999363\\
3	5.65093465503908\\
3	5.65166921938108\\
3	5.66667206656629\\
3	5.75824533237168\\
3	5.79909052050896\\
3	5.84397374519663\\
3	5.94786436525512\\
3	5.98027672359042\\
3	6.05838639033894\\
3	6.33150604382338\\
3	6.34308297470479\\
3	6.36225172124121\\
3	6.42239320747741\\
3	6.4614409918919\\
3	6.69589311805412\\
3	6.8401248030298\\
3	6.96554698612788\\
3	6.96659143711054\\
3	7.00646449157777\\
3	7.20038561833635\\
3	7.28307833227922\\
3	7.37153660132348\\
3	8.66212494665384\\
};
\addplot [color=black,mark size=0.5pt,only marks,mark=*,mark options={solid,fill=black,draw=black},forget plot]
  table[row sep=crcr]{%
4	5.56017231805677\\
4	5.5962939183071\\
4	5.7152532891024\\
4	5.73755054868447\\
4	5.76535567528108\\
4	5.79619673211197\\
4	5.98129809708458\\
4	5.9982594967525\\
4	6.10268461237068\\
4	6.16750822155245\\
4	6.45226525171938\\
4	6.74905251831213\\
4	7.16114349410359\\
4	7.36781484390658\\
4	7.70941305054388\\
4	8.11239644883154\\
4	8.49875059421999\\
4	9.30474107432021\\
4	9.36787189963635\\
4	9.94910464518864\\
4	10.1199357641457\\
4	10.5538887122452\\
4	10.8970545802976\\
};
\addplot [color=black,mark size=0.5pt,only marks,mark=*,mark options={solid,fill=black,draw=black},forget plot]
  table[row sep=crcr]{%
5	4.45459994655597\\
5	4.55880220861135\\
5	4.5903915165324\\
5	4.6134390345471\\
5	4.70726358520439\\
5	4.76345378861824\\
5	5.04462375593958\\
5	5.34067060384267\\
5	5.41094497839934\\
5	5.45589648514484\\
5	5.64825718542482\\
5	5.69290962031444\\
5	6.91550328413651\\
5	7.7805467742869\\
};
\addplot [color=black,mark size=0.5pt,only marks,mark=*,mark options={solid,fill=black,draw=black},forget plot]
  table[row sep=crcr]{%
6	4.36682771182719\\
6	4.38092800267321\\
6	4.39026607207259\\
6	4.4328871551004\\
6	4.43945050817574\\
6	4.45147424753556\\
6	4.56498848083824\\
6	4.61529498642377\\
6	4.87746929903191\\
6	4.89924936711811\\
6	5.75279342791843\\
6	6.15287677550461\\
6	6.18218126621647\\
6	6.22297686604729\\
6	6.3251858984485\\
6	6.36637290510273\\
6	9.78678806156092\\
};
\addplot [color=black,mark size=0.5pt,only marks,mark=*,mark options={solid,fill=black,draw=black},forget plot]
  table[row sep=crcr]{%
7	1.42071702776304\\
7	5.46543326704899\\
7	5.49201054773347\\
7	5.50935555325352\\
7	5.53055086338449\\
7	5.62788449542153\\
7	5.65568501881566\\
7	5.67478572516651\\
7	5.73190412159585\\
7	5.7590720770441\\
7	5.76589643984145\\
7	5.84700533227056\\
7	5.87885484414587\\
7	6.4599195293154\\
7	6.55046397190353\\
7	6.7245181793853\\
7	6.7526214689269\\
7	6.84857212347617\\
7	7.14632209518174\\
7	7.27456370941551\\
7	7.35285305303799\\
7	7.8820887401112\\
7	10.0202834461502\\
};
\addplot [color=black,mark size=0.5pt,only marks,mark=*,mark options={solid,fill=black,draw=black},forget plot]
  table[row sep=crcr]{%
8	5.40203977693024\\
8	5.53167470167511\\
8	5.54950692403795\\
8	5.61819739290132\\
8	5.67800075592853\\
8	5.70127940816217\\
8	5.71925341161961\\
8	5.72955095392235\\
8	5.74193078070767\\
8	5.74995025599702\\
8	5.7510091965317\\
8	5.75904882847488\\
8	5.84757657196876\\
8	5.86383036947865\\
8	5.89399899088992\\
8	5.96625764040468\\
8	5.98475387922107\\
8	6.01306191095537\\
8	6.06576803332538\\
8	6.09707523228396\\
8	6.11426988381807\\
8	6.11829686778094\\
8	6.15253121086838\\
8	6.1775846589\\
8	6.26767252649281\\
8	6.96188975741885\\
8	7.36982650025685\\
8	7.49903780859454\\
8	9.01273170420808\\
8	10.0128673623669\\
8	10.0474020420095\\
8	10.8742598970765\\
8	11.8701025834861\\
};
\end{axis}
\end{tikzpicture}%
%
 } \hfil
        \subfigure{%
%
%
\begin{tikzpicture}

\begin{axis}[%
width=3.35cm,
height=1.6cm,
unbounded coords=jump,
scale only axis,
separate axis lines,
every outer x axis line/.append style={white!15!black},
every x tick label/.append style={font=\color{white!15!black}},
xmin=0.5,
xmax=8.5,
xtick={1,2,3,4,5,6,7,8},
xticklabels={{PCA},{kPCA},{Varimax},{ICA},{SPCA},{SSPCA},{SPLOCS},{our}},
every outer y axis line/.append style={white!15!black},
every y tick label/.append style={font=\color{white!15!black}},
ymin=0.748116750476481,
ymax=2.6354777857299,
ylabel={$g_{\text{avg}}~[\text{mm}]$},
title style={font=\bfseries},
title={generalisation error},
every x tick label/.append style={rotate=45, anchor=east, align=left, font=\tiny}
]
\addplot [color=blue,solid,forget plot]
  table[row sep=crcr]{%
1	1.43503686579834\\
1	2.54968864776383\\
};
\addplot [color=blue,solid,forget plot]
  table[row sep=crcr]{%
2	0.833905888442545\\
2	1.34152677673347\\
};
\addplot [color=blue,solid,forget plot]
  table[row sep=crcr]{%
3	1.43503686579834\\
3	2.54968864776383\\
};
\addplot [color=blue,solid,forget plot]
  table[row sep=crcr]{%
4	1.43503686579834\\
4	2.54968864776383\\
};
\addplot [color=blue,solid,forget plot]
  table[row sep=crcr]{%
5	0.837899129483047\\
5	1.17068262892941\\
};
\addplot [color=blue,solid,forget plot]
  table[row sep=crcr]{%
6	0.927579041195051\\
6	1.4652743451596\\
};
\addplot [color=blue,solid,forget plot]
  table[row sep=crcr]{%
7	1.25673642254051\\
7	1.93938293517003\\
};
\addplot [color=blue,solid,forget plot]
  table[row sep=crcr]{%
8	0.929381662088435\\
8	1.45246110803437\\
};
\addplot [color=blue,solid,line width=4.0pt,forget plot]
  table[row sep=crcr]{%
1	1.79255981107551\\
1	2.33458398293861\\
};
\addplot [color=blue,solid,line width=4.0pt,forget plot]
  table[row sep=crcr]{%
2	0.949989376922774\\
2	1.18978435049488\\
};
\addplot [color=blue,solid,line width=4.0pt,forget plot]
  table[row sep=crcr]{%
3	1.79255981107551\\
3	2.33458398293861\\
};
\addplot [color=blue,solid,line width=4.0pt,forget plot]
  table[row sep=crcr]{%
4	1.79255981107551\\
4	2.33458398293861\\
};
\addplot [color=blue,solid,line width=4.0pt,forget plot]
  table[row sep=crcr]{%
5	0.890922908583421\\
5	1.04186220498246\\
};
\addplot [color=blue,solid,line width=4.0pt,forget plot]
  table[row sep=crcr]{%
6	1.05122813636988\\
6	1.27821876632231\\
};
\addplot [color=blue,solid,line width=4.0pt,forget plot]
  table[row sep=crcr]{%
7	1.32226494423058\\
7	1.75805206328313\\
};
\addplot [color=blue,solid,line width=4.0pt,forget plot]
  table[row sep=crcr]{%
8	1.06544074017437\\
8	1.27701852702536\\
};
\addplot [color=red,solid,line width=1.0pt,forget plot]
  table[row sep=crcr]{%
0.75	2.11611647614197\\
1.25	2.11611647614197\\
};
\addplot [color=red,solid,line width=1.0pt,forget plot]
  table[row sep=crcr]{%
1.75	0.973120552076018\\
2.25	0.973120552076018\\
};
\addplot [color=red,solid,line width=1.0pt,forget plot]
  table[row sep=crcr]{%
2.75	2.11611647614197\\
3.25	2.11611647614197\\
};
\addplot [color=red,solid,line width=1.0pt,forget plot]
  table[row sep=crcr]{%
3.75	2.11611647614198\\
4.25	2.11611647614198\\
};
\addplot [color=red,solid,line width=1.0pt,forget plot]
  table[row sep=crcr]{%
4.75	0.963379230681382\\
5.25	0.963379230681382\\
};
\addplot [color=red,solid,line width=1.0pt,forget plot]
  table[row sep=crcr]{%
5.75	1.21038123150902\\
6.25	1.21038123150902\\
};
\addplot [color=red,solid,line width=1.0pt,forget plot]
  table[row sep=crcr]{%
6.75	1.50449892502992\\
7.25	1.50449892502992\\
};
\addplot [color=red,solid,line width=1.0pt,forget plot]
  table[row sep=crcr]{%
7.75	1.12758605070021\\
8.25	1.12758605070021\\
};
\addplot [color=black,mark size=0.5pt,only marks,mark=*,mark options={solid,fill=black,draw=black},forget plot]
  table[row sep=crcr]{%
nan	nan\\
};
\addplot [color=black,mark size=0.5pt,only marks,mark=*,mark options={solid,fill=black,draw=black},forget plot]
  table[row sep=crcr]{%
nan	nan\\
};
\addplot [color=black,mark size=0.5pt,only marks,mark=*,mark options={solid,fill=black,draw=black},forget plot]
  table[row sep=crcr]{%
nan	nan\\
};
\addplot [color=black,mark size=0.5pt,only marks,mark=*,mark options={solid,fill=black,draw=black},forget plot]
  table[row sep=crcr]{%
nan	nan\\
};
\addplot [color=black,mark size=0.5pt,only marks,mark=*,mark options={solid,fill=black,draw=black},forget plot]
  table[row sep=crcr]{%
nan	nan\\
};
\addplot [color=black,mark size=0.5pt,only marks,mark=*,mark options={solid,fill=black,draw=black},forget plot]
  table[row sep=crcr]{%
6	1.71531425299899\\
};
\addplot [color=black,mark size=0.5pt,only marks,mark=*,mark options={solid,fill=black,draw=black},forget plot]
  table[row sep=crcr]{%
nan	nan\\
};
\addplot [color=black,mark size=0.5pt,only marks,mark=*,mark options={solid,fill=black,draw=black},forget plot]
  table[row sep=crcr]{%
nan	nan\\
};
\end{axis}
\end{tikzpicture}%
%
 } \hfil
        \subfigure{%
%
%
\begin{tikzpicture}

\begin{axis}[%
width=3.35cm,
height=1.6cm,
unbounded coords=jump,
scale only axis,
separate axis lines,
every outer x axis line/.append style={white!15!black},
every x tick label/.append style={font=\color{white!15!black}},
xmin=0.5,
xmax=8.5,
xtick={1,2,3,4,5,6,7,8},
xticklabels={{PCA},{kPCA},{Varimax},{ICA},{SPCA},{SSPCA},{SPLOCS},{our}},
every outer y axis line/.append style={white!15!black},
every y tick label/.append style={font=\color{white!15!black}},
ymin=1.0560020150112,
ymax=3.51864296233463,
ylabel={$g_{\text{avg}}^{0.05}~[\text{mm}]$},
title style={font=\bfseries},
title={sparse reconstruction error},
every x tick label/.append style={rotate=45, anchor=east, align=left, font=\tiny}
]
\addplot [color=blue,solid,forget plot]
  table[row sep=crcr]{%
1	1.4422273543522\\
1	2.64360871402138\\
};
\addplot [color=blue,solid,forget plot]
  table[row sep=crcr]{%
2	1.16794023988954\\
2	1.85245739540011\\
};
\addplot [color=blue,solid,forget plot]
  table[row sep=crcr]{%
3	1.4422273543522\\
3	2.64360871402138\\
};
\addplot [color=blue,solid,forget plot]
  table[row sep=crcr]{%
4	1.4422273543522\\
4	2.64360871402138\\
};
\addplot [color=blue,solid,forget plot]
  table[row sep=crcr]{%
5	1.18981329464009\\
5	2.61620129347959\\
};
\addplot [color=blue,solid,forget plot]
  table[row sep=crcr]{%
6	1.29095909947549\\
6	3.4067047374563\\
};
\addplot [color=blue,solid,forget plot]
  table[row sep=crcr]{%
7	1.43925976333935\\
7	2.62960424146664\\
};
\addplot [color=blue,solid,forget plot]
  table[row sep=crcr]{%
8	1.45116705042192\\
8	2.450526007045\\
};
\addplot [color=blue,solid,line width=4.0pt,forget plot]
  table[row sep=crcr]{%
1	1.82110603082426\\
1	2.36895103140003\\
};
\addplot [color=blue,solid,line width=4.0pt,forget plot]
  table[row sep=crcr]{%
2	1.26358814641349\\
2	1.57143545840145\\
};
\addplot [color=blue,solid,line width=4.0pt,forget plot]
  table[row sep=crcr]{%
3	1.82110603082426\\
3	2.36895103140003\\
};
\addplot [color=blue,solid,line width=4.0pt,forget plot]
  table[row sep=crcr]{%
4	1.82110603082426\\
4	2.36895103140003\\
};
\addplot [color=blue,solid,line width=4.0pt,forget plot]
  table[row sep=crcr]{%
5	1.7080918096062\\
5	2.17587321332228\\
};
\addplot [color=blue,solid,line width=4.0pt,forget plot]
  table[row sep=crcr]{%
6	1.80339926613254\\
6	2.57484565189368\\
};
\addplot [color=blue,solid,line width=4.0pt,forget plot]
  table[row sep=crcr]{%
7	1.79409587351531\\
7	2.34709152946341\\
};
\addplot [color=blue,solid,line width=4.0pt,forget plot]
  table[row sep=crcr]{%
8	1.81836506699204\\
8	2.24698502574387\\
};
\addplot [color=red,solid,line width=1.0pt,forget plot]
  table[row sep=crcr]{%
0.75	2.17101700814461\\
1.25	2.17101700814461\\
};
\addplot [color=red,solid,line width=1.0pt,forget plot]
  table[row sep=crcr]{%
1.75	1.30802464176432\\
2.25	1.30802464176432\\
};
\addplot [color=red,solid,line width=1.0pt,forget plot]
  table[row sep=crcr]{%
2.75	2.17101700814462\\
3.25	2.17101700814462\\
};
\addplot [color=red,solid,line width=1.0pt,forget plot]
  table[row sep=crcr]{%
3.75	2.17101700814461\\
4.25	2.17101700814461\\
};
\addplot [color=red,solid,line width=1.0pt,forget plot]
  table[row sep=crcr]{%
4.75	1.80882520730622\\
5.25	1.80882520730622\\
};
\addplot [color=red,solid,line width=1.0pt,forget plot]
  table[row sep=crcr]{%
5.75	2.00704215043387\\
6.25	2.00704215043387\\
};
\addplot [color=red,solid,line width=1.0pt,forget plot]
  table[row sep=crcr]{%
6.75	2.18297038709687\\
7.25	2.18297038709687\\
};
\addplot [color=red,solid,line width=1.0pt,forget plot]
  table[row sep=crcr]{%
7.75	2.11133061238332\\
8.25	2.11133061238332\\
};
\addplot [color=black,mark size=0.5pt,only marks,mark=*,mark options={solid,fill=black,draw=black},forget plot]
  table[row sep=crcr]{%
nan	nan\\
};
\addplot [color=black,mark size=0.5pt,only marks,mark=*,mark options={solid,fill=black,draw=black},forget plot]
  table[row sep=crcr]{%
nan	nan\\
};
\addplot [color=black,mark size=0.5pt,only marks,mark=*,mark options={solid,fill=black,draw=black},forget plot]
  table[row sep=crcr]{%
nan	nan\\
};
\addplot [color=black,mark size=0.5pt,only marks,mark=*,mark options={solid,fill=black,draw=black},forget plot]
  table[row sep=crcr]{%
nan	nan\\
};
\addplot [color=black,mark size=0.5pt,only marks,mark=*,mark options={solid,fill=black,draw=black},forget plot]
  table[row sep=crcr]{%
5	2.98047714426036\\
};
\addplot [color=black,mark size=0.5pt,only marks,mark=*,mark options={solid,fill=black,draw=black},forget plot]
  table[row sep=crcr]{%
nan	nan\\
};
\addplot [color=black,mark size=0.5pt,only marks,mark=*,mark options={solid,fill=black,draw=black},forget plot]
  table[row sep=crcr]{%
nan	nan\\
};
\addplot [color=black,mark size=0.5pt,only marks,mark=*,mark options={solid,fill=black,draw=black},forget plot]
  table[row sep=crcr]{%
nan	nan\\
};
\end{axis}
\end{tikzpicture}%
%
}
     } 
     \vspace{-5mm}
     \centerline{
        \subfigure{%
%
%
\begin{tikzpicture}

\begin{axis}[%
width=3.35cm,
height=1.6cm,
unbounded coords=jump,
scale only axis,
separate axis lines,
every outer x axis line/.append style={white!15!black},
every x tick label/.append style={font=\color{white!15!black}},
xmin=0.5,
xmax=8.5,
xtick={1,2,3,4,5,6,7,8},
xticklabels={{PCA},{kPCA},{Varimax},{ICA},{SPCA},{SSPCA},{SPLOCS},{our}},
every outer y axis line/.append style={white!15!black},
every y tick label/.append style={font=\color{white!15!black}},
ymin=-0.267056230617295,
ymax=5.60818084296354,
ylabel={$e_{\max}~[\text{mm}]$},
every x tick label/.append style={rotate=45, anchor=east, align=left, font=\tiny}
]
\addplot [color=blue,solid,forget plot]
  table[row sep=crcr]{%
1	1.62805793556083e-14\\
1	5.90221124448794e-14\\
};
\addplot [color=blue,solid,forget plot]
  table[row sep=crcr]{%
2	0.166741286163411\\
2	0.41537644663779\\
};
\addplot [color=blue,solid,forget plot]
  table[row sep=crcr]{%
3	1.92142370039363e-14\\
3	6.57791917336877e-14\\
};
\addplot [color=blue,solid,forget plot]
  table[row sep=crcr]{%
4	2.07157025572446e-14\\
4	9.2370555648813e-14\\
};
\addplot [color=blue,solid,forget plot]
  table[row sep=crcr]{%
5	1.84185688100246\\
5	2.77533076029395\\
};
\addplot [color=blue,solid,forget plot]
  table[row sep=crcr]{%
6	1.72662142532216\\
6	2.65370191605703\\
};
\addplot [color=blue,solid,forget plot]
  table[row sep=crcr]{%
7	0.906130275522009\\
7	1.71634421483041\\
};
\addplot [color=blue,solid,forget plot]
  table[row sep=crcr]{%
8	2.71236477377645\\
8	5.1896788837798\\
};
\addplot [color=blue,solid,line width=4.0pt,forget plot]
  table[row sep=crcr]{%
1	2.26962106730995e-14\\
1	3.84413924908129e-14\\
};
\addplot [color=blue,solid,line width=4.0pt,forget plot]
  table[row sep=crcr]{%
2	0.2363814129736\\
2	0.316674749146871\\
};
\addplot [color=blue,solid,line width=4.0pt,forget plot]
  table[row sep=crcr]{%
3	3.38906323110192e-14\\
3	5.49944460039743e-14\\
};
\addplot [color=blue,solid,line width=4.0pt,forget plot]
  table[row sep=crcr]{%
4	2.58561427636847e-14\\
4	6.82265835016356e-14\\
};
\addplot [color=blue,solid,line width=4.0pt,forget plot]
  table[row sep=crcr]{%
5	2.02540839893988\\
5	2.43954033454077\\
};
\addplot [color=blue,solid,line width=4.0pt,forget plot]
  table[row sep=crcr]{%
6	2.02798187078215\\
6	2.34412717479191\\
};
\addplot [color=blue,solid,line width=4.0pt,forget plot]
  table[row sep=crcr]{%
7	1.19773548892185\\
7	1.49277016931487\\
};
\addplot [color=blue,solid,line width=4.0pt,forget plot]
  table[row sep=crcr]{%
8	3.10844485150807\\
8	3.98887445100809\\
};
\addplot [color=red,solid,line width=1.0pt,forget plot]
  table[row sep=crcr]{%
0.75	2.95644699425659e-14\\
1.25	2.95644699425659e-14\\
};
\addplot [color=red,solid,line width=1.0pt,forget plot]
  table[row sep=crcr]{%
1.75	0.258911310507885\\
2.25	0.258911310507885\\
};
\addplot [color=red,solid,line width=1.0pt,forget plot]
  table[row sep=crcr]{%
2.75	3.70169068449661e-14\\
3.25	3.70169068449661e-14\\
};
\addplot [color=red,solid,line width=1.0pt,forget plot]
  table[row sep=crcr]{%
3.75	3.83051078018465e-14\\
4.25	3.83051078018465e-14\\
};
\addplot [color=red,solid,line width=1.0pt,forget plot]
  table[row sep=crcr]{%
4.75	2.22243250214085\\
5.25	2.22243250214085\\
};
\addplot [color=red,solid,line width=1.0pt,forget plot]
  table[row sep=crcr]{%
5.75	2.22744896341752\\
6.25	2.22744896341752\\
};
\addplot [color=red,solid,line width=1.0pt,forget plot]
  table[row sep=crcr]{%
6.75	1.33283312654954\\
7.25	1.33283312654954\\
};
\addplot [color=red,solid,line width=1.0pt,forget plot]
  table[row sep=crcr]{%
7.75	3.35237233182128\\
8.25	3.35237233182128\\
};
\addplot [color=black,mark size=0.5pt,only marks,mark=*,mark options={solid,fill=black,draw=black},forget plot]
  table[row sep=crcr]{%
nan	nan\\
};
\addplot [color=black,mark size=0.5pt,only marks,mark=*,mark options={solid,fill=black,draw=black},forget plot]
  table[row sep=crcr]{%
nan	nan\\
};
\addplot [color=black,mark size=0.5pt,only marks,mark=*,mark options={solid,fill=black,draw=black},forget plot]
  table[row sep=crcr]{%
3	8.70958464256778e-14\\
};
\addplot [color=black,mark size=0.5pt,only marks,mark=*,mark options={solid,fill=black,draw=black},forget plot]
  table[row sep=crcr]{%
nan	nan\\
};
\addplot [color=black,mark size=0.5pt,only marks,mark=*,mark options={solid,fill=black,draw=black},forget plot]
  table[row sep=crcr]{%
nan	nan\\
};
\addplot [color=black,mark size=0.5pt,only marks,mark=*,mark options={solid,fill=black,draw=black},forget plot]
  table[row sep=crcr]{%
nan	nan\\
};
\addplot [color=black,mark size=0.5pt,only marks,mark=*,mark options={solid,fill=black,draw=black},forget plot]
  table[row sep=crcr]{%
7	0.0848860297109979\\
7	1.94517101460767\\
};
\addplot [color=black,mark size=0.5pt,only marks,mark=*,mark options={solid,fill=black,draw=black},forget plot]
  table[row sep=crcr]{%
8	5.34112461234623\\
};
\end{axis}
\end{tikzpicture}%
%
 } \hfil
        \subfigure{%
%
%
\begin{tikzpicture}

\begin{axis}[%
width=3.35cm,
height=1.6cm,
scale only axis,
separate axis lines,
every outer x axis line/.append style={white!15!black},
every x tick label/.append style={font=\color{white!15!black}},
xmin=0.5,
xmax=8.5,
xtick={1,2,3,4,5,6,7,8},
xticklabels={{PCA},{kPCA},{Varimax},{ICA},{SPCA},{SSPCA},{SPLOCS},{our}},
every outer y axis line/.append style={white!15!black},
every y tick label/.append style={font=\color{white!15!black}},
ymin=1.72295853920056,
ymax=29.8726154803725,
ylabel={$s_{\max}~[\text{mm}]$},
every x tick label/.append style={rotate=45, anchor=east, align=left, font=\tiny}
]
\addplot [color=blue,solid,forget plot]
  table[row sep=crcr]{%
1	4.21947764929941\\
1	11.4220974576402\\
};
\addplot [color=blue,solid,forget plot]
  table[row sep=crcr]{%
2	4.19817692128453\\
2	11.3327632955004\\
};
\addplot [color=blue,solid,forget plot]
  table[row sep=crcr]{%
3	4.41597864312842\\
3	11.0382080543337\\
};
\addplot [color=blue,solid,forget plot]
  table[row sep=crcr]{%
4	4.06169834294779\\
4	11.3226311165512\\
};
\addplot [color=blue,solid,forget plot]
  table[row sep=crcr]{%
5	4.17830770070483\\
5	9.3808480447854\\
};
\addplot [color=blue,solid,forget plot]
  table[row sep=crcr]{%
6	3.93495841190694\\
6	9.3998427654441\\
};
\addplot [color=blue,solid,forget plot]
  table[row sep=crcr]{%
7	4.43746046424502\\
7	11.4434702817392\\
};
\addplot [color=blue,solid,forget plot]
  table[row sep=crcr]{%
8	4.33191211580419\\
8	10.7315009915054\\
};
\addplot [color=blue,solid,line width=4.0pt,forget plot]
  table[row sep=crcr]{%
1	6.82888883210361\\
1	8.67254612840264\\
};
\addplot [color=blue,solid,line width=4.0pt,forget plot]
  table[row sep=crcr]{%
2	6.83420991796562\\
2	8.63945726947865\\
};
\addplot [color=blue,solid,line width=4.0pt,forget plot]
  table[row sep=crcr]{%
3	6.83066655049112\\
3	8.51864920412558\\
};
\addplot [color=blue,solid,line width=4.0pt,forget plot]
  table[row sep=crcr]{%
4	6.78704554847206\\
4	8.63830023109156\\
};
\addplot [color=blue,solid,line width=4.0pt,forget plot]
  table[row sep=crcr]{%
5	6.05433608538822\\
5	7.44259877194215\\
};
\addplot [color=blue,solid,line width=4.0pt,forget plot]
  table[row sep=crcr]{%
6	5.91447750644192\\
6	7.31640847790602\\
};
\addplot [color=blue,solid,line width=4.0pt,forget plot]
  table[row sep=crcr]{%
7	6.94549422227539\\
7	8.74694136450085\\
};
\addplot [color=blue,solid,line width=4.0pt,forget plot]
  table[row sep=crcr]{%
8	6.67625735148239\\
8	8.32603933641715\\
};
\addplot [color=red,solid,line width=1.0pt,forget plot]
  table[row sep=crcr]{%
0.75	7.68025533631348\\
1.25	7.68025533631348\\
};
\addplot [color=red,solid,line width=1.0pt,forget plot]
  table[row sep=crcr]{%
1.75	7.74918625934548\\
2.25	7.74918625934548\\
};
\addplot [color=red,solid,line width=1.0pt,forget plot]
  table[row sep=crcr]{%
2.75	7.69711543406759\\
3.25	7.69711543406759\\
};
\addplot [color=red,solid,line width=1.0pt,forget plot]
  table[row sep=crcr]{%
3.75	7.71305033324182\\
4.25	7.71305033324182\\
};
\addplot [color=red,solid,line width=1.0pt,forget plot]
  table[row sep=crcr]{%
4.75	6.69174582358179\\
5.25	6.69174582358179\\
};
\addplot [color=red,solid,line width=1.0pt,forget plot]
  table[row sep=crcr]{%
5.75	6.60942997561947\\
6.25	6.60942997561947\\
};
\addplot [color=red,solid,line width=1.0pt,forget plot]
  table[row sep=crcr]{%
6.75	7.8316490452718\\
7.25	7.8316490452718\\
};
\addplot [color=red,solid,line width=1.0pt,forget plot]
  table[row sep=crcr]{%
7.75	7.48604724664141\\
8.25	7.48604724664141\\
};
\addplot [color=black,mark size=0.5pt,only marks,mark=*,mark options={solid,fill=black,draw=black},forget plot]
  table[row sep=crcr]{%
1	3.00248840016292\\
1	3.98784078202974\\
1	11.4685246861344\\
1	11.5613459845257\\
1	11.5887094694494\\
1	11.5985258828704\\
1	11.6556983916887\\
1	11.7543582563204\\
1	12.3244061370162\\
1	12.4395397481343\\
1	12.4962253856488\\
1	12.5341866935799\\
1	12.538233461871\\
1	12.7067799034227\\
1	13.0690709881822\\
1	13.1674055851914\\
1	13.3269219410583\\
1	13.8901315258154\\
1	15.4073590633657\\
1	28.5930856194102\\
};
\addplot [color=black,mark size=0.5pt,only marks,mark=*,mark options={solid,fill=black,draw=black},forget plot]
  table[row sep=crcr]{%
2	11.4479906528872\\
2	11.6193910775525\\
2	11.6309423987935\\
2	11.7031718185096\\
2	11.8198008358459\\
2	11.8522406662437\\
2	12.0018732957821\\
2	12.477305401674\\
2	12.8442656283975\\
2	13.3744169440639\\
2	13.6633407248883\\
2	13.7475579041745\\
2	14.06597085588\\
2	14.5743607019246\\
2	14.7100426677707\\
2	15.2177835342632\\
};
\addplot [color=black,mark size=0.5pt,only marks,mark=*,mark options={solid,fill=black,draw=black},forget plot]
  table[row sep=crcr]{%
3	3.4846023903706\\
3	3.53941071257036\\
3	3.80424833919754\\
3	4.09154487283732\\
3	11.0521261108607\\
3	11.0695347531877\\
3	11.1900073636038\\
3	11.4639901350778\\
3	11.5832178337937\\
3	11.6087775048094\\
3	11.9574336284391\\
3	11.9587507869115\\
3	11.9628806901227\\
3	12.0132868915842\\
3	12.1943344046603\\
3	12.2105511105161\\
3	12.2167015502568\\
3	12.2353809056642\\
3	12.3486165973295\\
3	12.4255903867195\\
3	12.5135292832472\\
3	12.6759108129271\\
3	12.9468947025485\\
3	13.0072644930125\\
3	13.5548100229987\\
3	13.7660266694126\\
3	15.7273811061605\\
};
\addplot [color=black,mark size=0.5pt,only marks,mark=*,mark options={solid,fill=black,draw=black},forget plot]
  table[row sep=crcr]{%
4	3.47669433226667\\
4	11.4371157005369\\
4	11.4962283895974\\
4	11.8096117530214\\
4	12.10885904375\\
4	12.4492077325717\\
4	12.6646798396603\\
4	12.9040709260385\\
4	13.1178979469795\\
4	14.0141498910694\\
4	14.1116435282035\\
4	14.2678233008784\\
4	15.2363163689662\\
4	15.4939067082125\\
4	15.5723171641975\\
4	15.6109083620022\\
4	15.8517291946422\\
4	19.4316684455299\\
};
\addplot [color=black,mark size=0.5pt,only marks,mark=*,mark options={solid,fill=black,draw=black},forget plot]
  table[row sep=crcr]{%
5	3.90053076486407\\
5	9.73524409191426\\
5	9.8154506493392\\
5	10.0079620153031\\
5	10.0700743559656\\
5	10.0919482822485\\
5	10.1153260118048\\
5	10.246395109949\\
5	10.6076170787037\\
5	10.8075429686182\\
5	10.9537353052967\\
5	11.215373564066\\
5	11.2605304608454\\
5	11.8961493834459\\
5	13.2260566806671\\
};
\addplot [color=black,mark size=0.5pt,only marks,mark=*,mark options={solid,fill=black,draw=black},forget plot]
  table[row sep=crcr]{%
6	3.56202608290803\\
6	9.51043039534377\\
6	9.59648810179751\\
6	9.74812261658999\\
6	10.0253808072352\\
6	10.2047018055693\\
6	10.2669754462955\\
6	10.2760840710431\\
6	10.4677293674444\\
6	10.5108420815156\\
6	10.5884824082392\\
6	11.5185506620292\\
6	13.3499681798537\\
6	13.9992025132902\\
6	14.9916586584185\\
};
\addplot [color=black,mark size=0.5pt,only marks,mark=*,mark options={solid,fill=black,draw=black},forget plot]
  table[row sep=crcr]{%
7	11.4720922469671\\
7	11.4738408930575\\
7	11.5450152273213\\
7	11.62167737505\\
7	11.6895854563648\\
7	11.7168448474095\\
7	11.7266956532567\\
7	11.7569265741879\\
7	11.7817482425394\\
7	11.9461160250753\\
7	11.9561547611273\\
7	12.0440980521181\\
7	12.0581049855945\\
7	12.0911212060068\\
7	12.1927339025343\\
7	12.3813636583751\\
7	13.4254802655743\\
7	13.6694917652973\\
7	13.7903912789349\\
7	14.9832781397199\\
7	15.0637680089115\\
7	18.176004603685\\
};
\addplot [color=black,mark size=0.5pt,only marks,mark=*,mark options={solid,fill=black,draw=black},forget plot]
  table[row sep=crcr]{%
8	10.9438220985725\\
8	10.9712762269262\\
8	10.9842573370228\\
8	10.9941854673194\\
8	11.0887142146244\\
8	11.1672118093463\\
8	11.2831448159793\\
8	11.4371488238054\\
8	11.5901277147874\\
8	11.6420546560299\\
8	11.6515115282454\\
8	11.8026692797082\\
8	11.8037566170857\\
8	11.8737952694738\\
8	12.1077432980617\\
8	12.3104072803095\\
8	12.3783915724532\\
8	12.458302858322\\
8	12.6822239118558\\
8	12.8124307658858\\
8	13.7043639050768\\
8	13.7990378527258\\
8	14.4482759397357\\
8	15.6214999917707\\
8	15.8645156293931\\
8	16.6085288396959\\
8	18.7770871785936\\
8	18.9750408728935\\
8	19.5357756035055\\
};
\end{axis}
\end{tikzpicture}%
%
 } \hfil
        \subfigure{%
%
%
\begin{tikzpicture}

\begin{axis}[%
width=3.35cm,
height=1.6cm,
unbounded coords=jump,
scale only axis,
separate axis lines,
every outer x axis line/.append style={white!15!black},
every x tick label/.append style={font=\color{white!15!black}},
xmin=0.5,
xmax=8.5,
xtick={1,2,3,4,5,6,7,8},
xticklabels={{PCA},{kPCA},{Varimax},{ICA},{SPCA},{SSPCA},{SPLOCS},{our}},
every outer y axis line/.append style={white!15!black},
every y tick label/.append style={font=\color{white!15!black}},
ymin=2.32499193318709,
ymax=9.99692018565619,
ylabel={$g_{\max}~[\text{mm}]$},
every x tick label/.append style={rotate=45, anchor=east, align=left, font=\tiny}
]
\addplot [color=blue,solid,forget plot]
  table[row sep=crcr]{%
1	4.25574351873048\\
1	8.24661610153002\\
};
\addplot [color=blue,solid,forget plot]
  table[row sep=crcr]{%
2	2.89721051759932\\
2	5.04129861837353\\
};
\addplot [color=blue,solid,forget plot]
  table[row sep=crcr]{%
3	4.25574351873049\\
3	8.24661610153002\\
};
\addplot [color=blue,solid,forget plot]
  table[row sep=crcr]{%
4	4.25574351873048\\
4	8.24661610153002\\
};
\addplot [color=blue,solid,forget plot]
  table[row sep=crcr]{%
5	2.87710876159548\\
5	5.07947506288235\\
};
\addplot [color=blue,solid,forget plot]
  table[row sep=crcr]{%
6	2.76043450498367\\
6	7.08464219604579\\
};
\addplot [color=blue,solid,forget plot]
  table[row sep=crcr]{%
7	3.7468673293346\\
7	7.46283766012515\\
};
\addplot [color=blue,solid,forget plot]
  table[row sep=crcr]{%
8	3.1149474043138\\
8	5.69334433323908\\
};
\addplot [color=blue,solid,line width=4.0pt,forget plot]
  table[row sep=crcr]{%
1	5.13704135741573\\
1	7.30108111694367\\
};
\addplot [color=blue,solid,line width=4.0pt,forget plot]
  table[row sep=crcr]{%
2	3.04095627409136\\
2	4.37445658177944\\
};
\addplot [color=blue,solid,line width=4.0pt,forget plot]
  table[row sep=crcr]{%
3	5.13704135741574\\
3	7.30108111694368\\
};
\addplot [color=blue,solid,line width=4.0pt,forget plot]
  table[row sep=crcr]{%
4	5.13704135741575\\
4	7.30108111694366\\
};
\addplot [color=blue,solid,line width=4.0pt,forget plot]
  table[row sep=crcr]{%
5	3.06725228800655\\
5	3.87689873616962\\
};
\addplot [color=blue,solid,line width=4.0pt,forget plot]
  table[row sep=crcr]{%
6	3.72532456876926\\
6	5.34170970277842\\
};
\addplot [color=blue,solid,line width=4.0pt,forget plot]
  table[row sep=crcr]{%
7	4.62244024620529\\
7	6.23315626890476\\
};
\addplot [color=blue,solid,line width=4.0pt,forget plot]
  table[row sep=crcr]{%
8	3.36429465593034\\
8	5.01506478928366\\
};
\addplot [color=red,solid,line width=1.0pt,forget plot]
  table[row sep=crcr]{%
0.75	6.14076961366918\\
1.25	6.14076961366918\\
};
\addplot [color=red,solid,line width=1.0pt,forget plot]
  table[row sep=crcr]{%
1.75	3.21427997397137\\
2.25	3.21427997397137\\
};
\addplot [color=red,solid,line width=1.0pt,forget plot]
  table[row sep=crcr]{%
2.75	6.14076961366919\\
3.25	6.14076961366919\\
};
\addplot [color=red,solid,line width=1.0pt,forget plot]
  table[row sep=crcr]{%
3.75	6.14076961366919\\
4.25	6.14076961366919\\
};
\addplot [color=red,solid,line width=1.0pt,forget plot]
  table[row sep=crcr]{%
4.75	3.36052942736974\\
5.25	3.36052942736974\\
};
\addplot [color=red,solid,line width=1.0pt,forget plot]
  table[row sep=crcr]{%
5.75	4.89870333149503\\
6.25	4.89870333149503\\
};
\addplot [color=red,solid,line width=1.0pt,forget plot]
  table[row sep=crcr]{%
6.75	5.27733212978877\\
7.25	5.27733212978877\\
};
\addplot [color=red,solid,line width=1.0pt,forget plot]
  table[row sep=crcr]{%
7.75	3.60041883250098\\
8.25	3.60041883250098\\
};
\addplot [color=black,mark size=0.5pt,only marks,mark=*,mark options={solid,fill=black,draw=black},forget plot]
  table[row sep=crcr]{%
nan	nan\\
};
\addplot [color=black,mark size=0.5pt,only marks,mark=*,mark options={solid,fill=black,draw=black},forget plot]
  table[row sep=crcr]{%
nan	nan\\
};
\addplot [color=black,mark size=0.5pt,only marks,mark=*,mark options={solid,fill=black,draw=black},forget plot]
  table[row sep=crcr]{%
nan	nan\\
};
\addplot [color=black,mark size=0.5pt,only marks,mark=*,mark options={solid,fill=black,draw=black},forget plot]
  table[row sep=crcr]{%
nan	nan\\
};
\addplot [color=black,mark size=0.5pt,only marks,mark=*,mark options={solid,fill=black,draw=black},forget plot]
  table[row sep=crcr]{%
5	6.59234093595512\\
};
\addplot [color=black,mark size=0.5pt,only marks,mark=*,mark options={solid,fill=black,draw=black},forget plot]
  table[row sep=crcr]{%
nan	nan\\
};
\addplot [color=black,mark size=0.5pt,only marks,mark=*,mark options={solid,fill=black,draw=black},forget plot]
  table[row sep=crcr]{%
7	9.64819617418032\\
};
\addplot [color=black,mark size=0.5pt,only marks,mark=*,mark options={solid,fill=black,draw=black},forget plot]
  table[row sep=crcr]{%
nan	nan\\
};
\end{axis}
\end{tikzpicture}%
%
 } \hfil
        \subfigure{%
%
%
\begin{tikzpicture}

\begin{axis}[%
width=3.35cm,
height=1.6cm,
unbounded coords=jump,
scale only axis,
separate axis lines,
every outer x axis line/.append style={white!15!black},
every x tick label/.append style={font=\color{white!15!black}},
xmin=0.5,
xmax=8.5,
xtick={1,2,3,4,5,6,7,8},
xticklabels={{PCA},{kPCA},{Varimax},{ICA},{SPCA},{SSPCA},{SPLOCS},{our}},
every outer y axis line/.append style={white!15!black},
every y tick label/.append style={font=\color{white!15!black}},
ymin=1.88067163867391,
ymax=53.7104327475409,
ylabel={$g_{\text{max}}^{0.05}~[\text{mm}]$},
every x tick label/.append style={rotate=45, anchor=east, align=left, font=\tiny}
]
\addplot [color=blue,solid,forget plot]
  table[row sep=crcr]{%
1	4.42582782840038\\
1	8.430645786412\\
};
\addplot [color=blue,solid,forget plot]
  table[row sep=crcr]{%
2	4.57144895064202\\
2	9.54677738961727\\
};
\addplot [color=blue,solid,forget plot]
  table[row sep=crcr]{%
3	4.42582782840038\\
3	8.43064578641201\\
};
\addplot [color=blue,solid,forget plot]
  table[row sep=crcr]{%
4	4.42582782840038\\
4	8.430645786412\\
};
\addplot [color=blue,solid,forget plot]
  table[row sep=crcr]{%
5	6.58042659207745\\
5	50.6388489853765\\
};
\addplot [color=blue,solid,forget plot]
  table[row sep=crcr]{%
6	5.01068742493945\\
6	51.3545345153196\\
};
\addplot [color=blue,solid,forget plot]
  table[row sep=crcr]{%
7	4.34890272624497\\
7	8.42713548914266\\
};
\addplot [color=blue,solid,forget plot]
  table[row sep=crcr]{%
8	4.23656987089514\\
8	8.20288641063637\\
};
\addplot [color=blue,solid,line width=4.0pt,forget plot]
  table[row sep=crcr]{%
1	5.30311204490036\\
1	7.67371435915983\\
};
\addplot [color=blue,solid,line width=4.0pt,forget plot]
  table[row sep=crcr]{%
2	5.1808574060051\\
2	7.82599727042616\\
};
\addplot [color=blue,solid,line width=4.0pt,forget plot]
  table[row sep=crcr]{%
3	5.30311204490035\\
3	7.67371435915983\\
};
\addplot [color=blue,solid,line width=4.0pt,forget plot]
  table[row sep=crcr]{%
4	5.30311204490035\\
4	7.67371435915984\\
};
\addplot [color=blue,solid,line width=4.0pt,forget plot]
  table[row sep=crcr]{%
5	10.2950334599346\\
5	33.4228043388408\\
};
\addplot [color=blue,solid,line width=4.0pt,forget plot]
  table[row sep=crcr]{%
6	10.2891867265694\\
6	27.1983359826462\\
};
\addplot [color=blue,solid,line width=4.0pt,forget plot]
  table[row sep=crcr]{%
7	5.41181639584691\\
7	7.72372952352475\\
};
\addplot [color=blue,solid,line width=4.0pt,forget plot]
  table[row sep=crcr]{%
8	4.85301212901969\\
8	7.35537478947997\\
};
\addplot [color=red,solid,line width=1.0pt,forget plot]
  table[row sep=crcr]{%
0.75	6.3866473900354\\
1.25	6.3866473900354\\
};
\addplot [color=red,solid,line width=1.0pt,forget plot]
  table[row sep=crcr]{%
1.75	6.22767712150524\\
2.25	6.22767712150524\\
};
\addplot [color=red,solid,line width=1.0pt,forget plot]
  table[row sep=crcr]{%
2.75	6.3866473900354\\
3.25	6.3866473900354\\
};
\addplot [color=red,solid,line width=1.0pt,forget plot]
  table[row sep=crcr]{%
3.75	6.3866473900354\\
4.25	6.3866473900354\\
};
\addplot [color=red,solid,line width=1.0pt,forget plot]
  table[row sep=crcr]{%
4.75	16.598526927864\\
5.25	16.598526927864\\
};
\addplot [color=red,solid,line width=1.0pt,forget plot]
  table[row sep=crcr]{%
5.75	20.4817561972355\\
6.25	20.4817561972355\\
};
\addplot [color=red,solid,line width=1.0pt,forget plot]
  table[row sep=crcr]{%
6.75	6.67484147197303\\
7.25	6.67484147197303\\
};
\addplot [color=red,solid,line width=1.0pt,forget plot]
  table[row sep=crcr]{%
7.75	6.15518560424921\\
8.25	6.15518560424921\\
};
\addplot [color=black,mark size=0.5pt,only marks,mark=*,mark options={solid,fill=black,draw=black},forget plot]
  table[row sep=crcr]{%
nan	nan\\
};
\addplot [color=black,mark size=0.5pt,only marks,mark=*,mark options={solid,fill=black,draw=black},forget plot]
  table[row sep=crcr]{%
nan	nan\\
};
\addplot [color=black,mark size=0.5pt,only marks,mark=*,mark options={solid,fill=black,draw=black},forget plot]
  table[row sep=crcr]{%
nan	nan\\
};
\addplot [color=black,mark size=0.5pt,only marks,mark=*,mark options={solid,fill=black,draw=black},forget plot]
  table[row sep=crcr]{%
nan	nan\\
};
\addplot [color=black,mark size=0.5pt,only marks,mark=*,mark options={solid,fill=black,draw=black},forget plot]
  table[row sep=crcr]{%
nan	nan\\
};
\addplot [color=black,mark size=0.5pt,only marks,mark=*,mark options={solid,fill=black,draw=black},forget plot]
  table[row sep=crcr]{%
nan	nan\\
};
\addplot [color=black,mark size=0.5pt,only marks,mark=*,mark options={solid,fill=black,draw=black},forget plot]
  table[row sep=crcr]{%
nan	nan\\
};
\addplot [color=black,mark size=0.5pt,only marks,mark=*,mark options={solid,fill=black,draw=black},forget plot]
  table[row sep=crcr]{%
nan	nan\\
};
\end{axis}
\end{tikzpicture}%
%
}
     }
     \vspace{-5mm}
    \caption{Boxplots of the quantitative measures for the brain shapes dataset. In each plot, the horizontal axis shows the individual methods and the vertical axis the error scores described in section \ref{quanmeas}.
    Compared to SPLOCS, which is the only other method explicitly striving for local support deformation factors, our method has a smaller generalisation error and sparse reconstruction error. The sparse but not spatially localised factors obtained by SPCA and SSPCA (cf.~Fig.~\ref{brainDefFactors}) result in a large maximum sparse reconstruction error (bottom right).
    } 
    \label{brain-boxplots}
\end{figure*} 

\subsubsection{Dealing with the Small Training Set}\label{kernelisation}
For PCA, Varimax and ICA the number of factors cannot exceed the rank of $\matX$, which is at most $K{-}1$ for $K{<}3N$. 
For the used matrix factorisation framework, setting $M$ to a value larger than the expected rank of the solution but smaller than full rank has empirically led to good results \cite{Haeffele:2014wj}. However, since our expected rank is $M=96$ and the full rank is at most $K{-}1=16$, this is not possible. 

We compensate the insufficient amount of training data by assuming smoothness of the deformations.
Based on concepts introduced in \cite{Cootes:1995uz,Wang:1998bw,Wang:2000iw}, instead of the data matrix $\matX$, we factorise the \emph{kernelised covariance} matrix $\matK$. The standard PCA method finds the $M$ most dominant eigenvectors of the covariance matrix $\matC$ by the (exact) factorisation $\matC = \matPhi \diag(\lambda_1, \ldots, \lambda_{K-1}) \matPhi^T$. The kernelised covariance $\matK$ allows to incorporate additional elasticity into the resulting deformation model. For example, $\matK {=} \matI_{3N}$ results in independent vertex movements \cite{Wang:2000iw}. A more interesting approach is to combine the (scaled) covariance matrix with a smooth vertex deformation. We define $\matK = \frac{1}{\|\VEC(\matX\matX^T)\|_\infty}\matX\matX^T + \matK_{\text{euc}}$, with $\matK_{\text{euc}} = \matI_3 \otimes \matK_{\text{euc}}' \in \R^{3N \times 3N}$. Using the bandwidth $\beta$, $\matK_{\text{euc}}'$ is given by
\begin{align}
   (\matK_{\text{euc}}')_{ij} = \exp(-(\frac{(\matD_{\text{euc}})_{ij}}{2 \beta \| \VEC(\matD_{\text{euc}})\|_\infty})^2) \,.
 \end{align} 
Setting $\matPhi$ to the $M$ most dominant eigenvectors of the symmetric and positive semi-definite matrix $\matK$ gives the solution of kPCA \cite{Bishop:2006ug}. In terms of our proposed matrix factorisation problem in eq.~\eqref{optProb}, we find a factorisation $\matPhi\tilde\matA$ of $\matK$, instead of factorising the data $\matX$. Since the regularisation term remains unchanged, the factor matrix $\matPhi {\in} \R^{3N \times M}$ is still sparse and smooth (due to $\| \cdot \|_\Phi$). Moreover, due to the nuclear norm term being contained in the product $\|\cdot\|_\Phi \|\cdot\|_A$ (cf.~section \ref{theomot}), the resulting factorisation $\matPhi\tilde\matA$ is steered towards being low-rank, in favour of the elaborations in section \ref{theomot}. However, the resulting matrix $\tilde\matA {\in} \R^{M \times 3N}$ now contains the weights for approximating $\matK$ by a linear combination of the columns of $\matPhi$, rather than approximating the data matrix $\matX$ by a linear combination of the columns of $\matPhi$. For the known $\matPhi$, the weights that best approximate the \emph{data matrix} $\matX$ are found by solving the linear system $\matPhi \matA {=} \matX$ in the least-squares sense for $\matA {\in} \R^{M \times K}$.

\subsubsection{Results}
The magnitude of the deformation of the first three factors are shown in Fig.~\ref{brainDefFactors}, where it can be seen that only SPCA, SSPCA, SPLOCS and our method obtain \emph{sparse} deformation factors. The SPCA and SSPCA methods do not incorporate the spatial relation between vertices and as such the deformations are not spatially localised (see red arrows in Fig.~\ref{brainDefFactors}, where more than a single region is active). The factors obtained by SPLOCS are non-smooth and do not exhibit local support, in contrast to our method, where smooth deformation factors with local support are obtained.

The quantitative results presented in Fig.~\ref{brain-boxplots} reveal that our method has a larger reconstruction error. This can be explained by the fact that due to the sparsity and smoothness of the deformation factors a very large number of basis vectors is required in order to exactly span the subspace of the training data. Instead, our method finds a simple (sparse and smooth) representation that explains the data approximately, in favour of Occam's razor. The average reconstruction error is around $1$mm, which is low considering that the brain structures span approximately $6$cm from left to right. Considering specificity, all methods are comparable. PCA, Varimax and ICA, which have the lowest reconstruction errors, have the highest generalisation errors, which underlines that these methods overfit the training data. The kPCA method is able to overcome this issue due to the smoothness assumption. SPCA and SSPCA have good generalisation scores but at the same time a very high maximum reconstruction error. Our method and SPLOCS are the only ones that explicitly strive for local support factors. Since our method outperforms SPLOCS with respect to generalisation and sparse reconstruction error, we claim that our method outperforms the state of the art.

\subsection{Human Body Shapes}
\begin{figure*}[ht!]
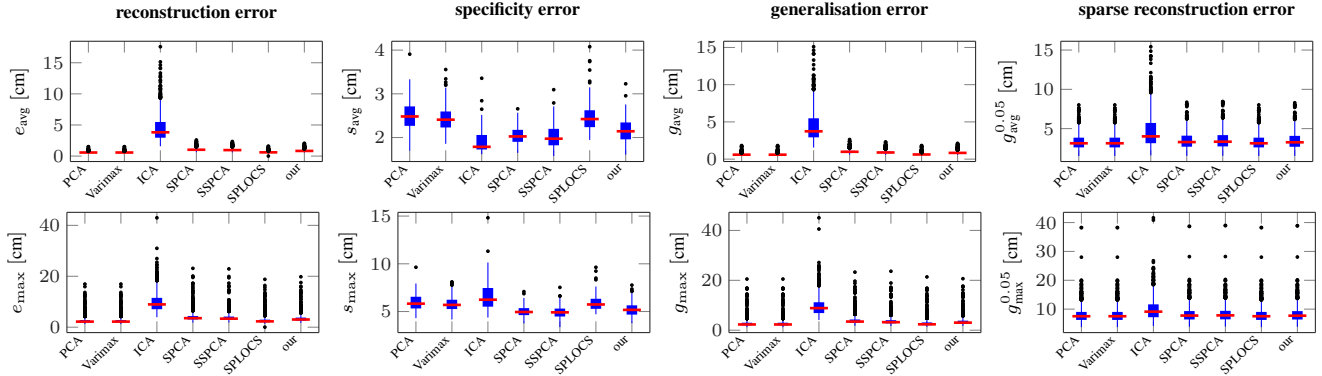

\vspace{-4mm}
     \centerline{
        \subfigure{%
    \input{boxplot-body-pdmReconstructionMeanError.tikz}%
 } \hfil
        \subfigure{%
%
%
\begin{tikzpicture}

\begin{axis}[%
width=3.35cm,
height=1.6cm,
scale only axis,
separate axis lines,
every outer x axis line/.append style={white!15!black},
every x tick label/.append style={font=\color{white!15!black}},
xmin=0.5,
xmax=7.5,
xtick={1,2,3,4,5,6,7},
xticklabels={{PCA},{Varimax},{ICA},{SPCA},{SSPCA},{SPLOCS},{our}},
every outer y axis line/.append style={white!15!black},
every y tick label/.append style={font=\color{white!15!black}},
ymin=1.45116950202736,
ymax=4.20237757735896,
ylabel={$s_{\text{avg}}~[\text{cm}]$},
title style={font=\bfseries},
title={specificity error},
every x tick label/.append style={rotate=45, anchor=east, align=left, font=\tiny}
]
\addplot [color=blue,solid,forget plot]
  table[row sep=crcr]{%
1	1.69922613980801\\
1	3.33153965978852\\
};
\addplot [color=blue,solid,forget plot]
  table[row sep=crcr]{%
2	1.85579991929912\\
2	3.14074568229829\\
};
\addplot [color=blue,solid,forget plot]
  table[row sep=crcr]{%
3	1.6188162113583\\
3	2.51858042160891\\
};
\addplot [color=blue,solid,forget plot]
  table[row sep=crcr]{%
4	1.65039128792935\\
4	2.5677345903623\\
};
\addplot [color=blue,solid,forget plot]
  table[row sep=crcr]{%
5	1.57622441454243\\
5	2.7115877968494\\
};
\addplot [color=blue,solid,forget plot]
  table[row sep=crcr]{%
6	1.94472267622454\\
6	3.14810322380093\\
};
\addplot [color=blue,solid,forget plot]
  table[row sep=crcr]{%
7	1.60923642749489\\
7	2.75664346501626\\
};
\addplot [color=blue,solid,line width=4.0pt,forget plot]
  table[row sep=crcr]{%
1	2.26739670052449\\
1	2.7122342038476\\
};
\addplot [color=blue,solid,line width=4.0pt,forget plot]
  table[row sep=crcr]{%
2	2.23028963198384\\
2	2.59582464900629\\
};
\addplot [color=blue,solid,line width=4.0pt,forget plot]
  table[row sep=crcr]{%
3	1.72657962274206\\
3	2.05853181297065\\
};
\addplot [color=blue,solid,line width=4.0pt,forget plot]
  table[row sep=crcr]{%
4	1.90167523337433\\
4	2.17568867720273\\
};
\addplot [color=blue,solid,line width=4.0pt,forget plot]
  table[row sep=crcr]{%
5	1.82228163574253\\
5	2.19851352677092\\
};
\addplot [color=blue,solid,line width=4.0pt,forget plot]
  table[row sep=crcr]{%
6	2.23768338738557\\
6	2.62373947518695\\
};
\addplot [color=blue,solid,line width=4.0pt,forget plot]
  table[row sep=crcr]{%
7	1.95874579035064\\
7	2.34478201995346\\
};
\addplot [color=red,solid,line width=1.0pt,forget plot]
  table[row sep=crcr]{%
0.75	2.48173792251049\\
1.25	2.48173792251049\\
};
\addplot [color=red,solid,line width=1.0pt,forget plot]
  table[row sep=crcr]{%
1.75	2.40822842417433\\
2.25	2.40822842417433\\
};
\addplot [color=red,solid,line width=1.0pt,forget plot]
  table[row sep=crcr]{%
2.75	1.78730624875241\\
3.25	1.78730624875241\\
};
\addplot [color=red,solid,line width=1.0pt,forget plot]
  table[row sep=crcr]{%
3.75	2.02521693076647\\
4.25	2.02521693076647\\
};
\addplot [color=red,solid,line width=1.0pt,forget plot]
  table[row sep=crcr]{%
4.75	1.97341485010004\\
5.25	1.97341485010004\\
};
\addplot [color=red,solid,line width=1.0pt,forget plot]
  table[row sep=crcr]{%
5.75	2.42243250167419\\
6.25	2.42243250167419\\
};
\addplot [color=red,solid,line width=1.0pt,forget plot]
  table[row sep=crcr]{%
6.75	2.14319886883715\\
7.25	2.14319886883715\\
};
\addplot [color=black,mark size=0.5pt,only marks,mark=*,mark options={solid,fill=black,draw=black},forget plot]
  table[row sep=crcr]{%
1	3.9070949654517\\
};
\addplot [color=black,mark size=0.5pt,only marks,mark=*,mark options={solid,fill=black,draw=black},forget plot]
  table[row sep=crcr]{%
2	3.19785772907502\\
2	3.21739204417464\\
2	3.26752896996035\\
2	3.34188017005798\\
2	3.55575072314439\\
};
\addplot [color=black,mark size=0.5pt,only marks,mark=*,mark options={solid,fill=black,draw=black},forget plot]
  table[row sep=crcr]{%
3	2.647667699936\\
3	2.84173258992089\\
3	3.35618966657144\\
};
\addplot [color=black,mark size=0.5pt,only marks,mark=*,mark options={solid,fill=black,draw=black},forget plot]
  table[row sep=crcr]{%
4	2.65649455424237\\
};
\addplot [color=black,mark size=0.5pt,only marks,mark=*,mark options={solid,fill=black,draw=black},forget plot]
  table[row sep=crcr]{%
5	2.79150400774309\\
5	3.09443387984667\\
};
\addplot [color=black,mark size=0.5pt,only marks,mark=*,mark options={solid,fill=black,draw=black},forget plot]
  table[row sep=crcr]{%
6	3.25371947758119\\
6	3.28868154531457\\
6	3.54334188639927\\
6	3.73836464455823\\
6	3.7570587669188\\
6	4.07732266484389\\
};
\addplot [color=black,mark size=0.5pt,only marks,mark=*,mark options={solid,fill=black,draw=black},forget plot]
  table[row sep=crcr]{%
7	2.95510256707136\\
7	3.2291106251962\\
};
\end{axis}
\end{tikzpicture}%
%
 } \hfil
        \subfigure{%
    \input{boxplot-body-generalisationMeanError.tikz}%
 } \hfil
        \subfigure{%
    \input{boxplot-body-meanErrorSparseFitting.tikz}%
}
     } 
     \vspace{-5mm}
     \centerline{
        \subfigure{%
    \input{boxplot-body-pdmReconstructionMaxError.tikz}%
 } \hfil
        \subfigure{%
%
%
\begin{tikzpicture}

\begin{axis}[%
width=3.35cm,
height=1.6cm,
scale only axis,
separate axis lines,
every outer x axis line/.append style={white!15!black},
every x tick label/.append style={font=\color{white!15!black}},
xmin=0.5,
xmax=7.5,
xtick={1,2,3,4,5,6,7},
xticklabels={{PCA},{Varimax},{ICA},{SPCA},{SSPCA},{SPLOCS},{our}},
every outer y axis line/.append style={white!15!black},
every y tick label/.append style={font=\color{white!15!black}},
ymin=2.79309648315672,
ymax=15.390354297186,
ylabel={$s_{\max}~[\text{cm}]$},
every x tick label/.append style={rotate=45, anchor=east, align=left, font=\tiny}
]
\addplot [color=blue,solid,forget plot]
  table[row sep=crcr]{%
1	4.30479416636364\\
1	7.93528299501422\\
};
\addplot [color=blue,solid,forget plot]
  table[row sep=crcr]{%
2	4.16984297155481\\
2	7.60602401652882\\
};
\addplot [color=blue,solid,forget plot]
  table[row sep=crcr]{%
3	4.36519998190626\\
3	10.137631999767\\
};
\addplot [color=blue,solid,forget plot]
  table[row sep=crcr]{%
4	3.76572250523721\\
4	6.42253481328886\\
};
\addplot [color=blue,solid,forget plot]
  table[row sep=crcr]{%
5	3.36569911106714\\
5	6.47260424471264\\
};
\addplot [color=blue,solid,forget plot]
  table[row sep=crcr]{%
6	4.71364258072073\\
6	7.62063556537264\\
};
\addplot [color=blue,solid,forget plot]
  table[row sep=crcr]{%
7	3.78944931584666\\
7	6.93531785095182\\
};
\addplot [color=blue,solid,line width=4.0pt,forget plot]
  table[row sep=crcr]{%
1	5.32354879356204\\
1	6.54690849836878\\
};
\addplot [color=blue,solid,line width=4.0pt,forget plot]
  table[row sep=crcr]{%
2	5.25914710549628\\
2	6.23621193990226\\
};
\addplot [color=blue,solid,line width=4.0pt,forget plot]
  table[row sep=crcr]{%
3	5.51216201280362\\
3	7.46541789773315\\
};
\addplot [color=blue,solid,line width=4.0pt,forget plot]
  table[row sep=crcr]{%
4	4.62504861399479\\
4	5.35601914792083\\
};
\addplot [color=blue,solid,line width=4.0pt,forget plot]
  table[row sep=crcr]{%
5	4.47920709671019\\
5	5.3092262637557\\
};
\addplot [color=blue,solid,line width=4.0pt,forget plot]
  table[row sep=crcr]{%
6	5.28332696365882\\
6	6.33650128449612\\
};
\addplot [color=blue,solid,line width=4.0pt,forget plot]
  table[row sep=crcr]{%
7	4.67867805289967\\
7	5.64905844348413\\
};
\addplot [color=red,solid,line width=1.0pt,forget plot]
  table[row sep=crcr]{%
0.75	5.82546959095411\\
1.25	5.82546959095411\\
};
\addplot [color=red,solid,line width=1.0pt,forget plot]
  table[row sep=crcr]{%
1.75	5.69369333815738\\
2.25	5.69369333815738\\
};
\addplot [color=red,solid,line width=1.0pt,forget plot]
  table[row sep=crcr]{%
2.75	6.24061031045541\\
3.25	6.24061031045541\\
};
\addplot [color=red,solid,line width=1.0pt,forget plot]
  table[row sep=crcr]{%
3.75	4.9502397615405\\
4.25	4.9502397615405\\
};
\addplot [color=red,solid,line width=1.0pt,forget plot]
  table[row sep=crcr]{%
4.75	4.91503000869147\\
5.25	4.91503000869147\\
};
\addplot [color=red,solid,line width=1.0pt,forget plot]
  table[row sep=crcr]{%
5.75	5.74034667795311\\
6.25	5.74034667795311\\
};
\addplot [color=red,solid,line width=1.0pt,forget plot]
  table[row sep=crcr]{%
6.75	5.17495491497787\\
7.25	5.17495491497787\\
};
\addplot [color=black,mark size=0.5pt,only marks,mark=*,mark options={solid,fill=black,draw=black},forget plot]
  table[row sep=crcr]{%
1	9.63836275077071\\
};
\addplot [color=black,mark size=0.5pt,only marks,mark=*,mark options={solid,fill=black,draw=black},forget plot]
  table[row sep=crcr]{%
2	7.75955269150287\\
2	7.83692940688833\\
2	7.91723947648058\\
2	8.06707083146468\\
2	8.10365795013798\\
};
\addplot [color=black,mark size=0.5pt,only marks,mark=*,mark options={solid,fill=black,draw=black},forget plot]
  table[row sep=crcr]{%
3	11.3171376005457\\
3	14.8177516692756\\
};
\addplot [color=black,mark size=0.5pt,only marks,mark=*,mark options={solid,fill=black,draw=black},forget plot]
  table[row sep=crcr]{%
4	6.86894526373073\\
4	7.08701452240076\\
};
\addplot [color=black,mark size=0.5pt,only marks,mark=*,mark options={solid,fill=black,draw=black},forget plot]
  table[row sep=crcr]{%
5	6.61595542574617\\
5	6.63680363901847\\
5	7.53049913439164\\
};
\addplot [color=black,mark size=0.5pt,only marks,mark=*,mark options={solid,fill=black,draw=black},forget plot]
  table[row sep=crcr]{%
6	8.32463645973721\\
6	8.49867448785763\\
6	9.233051410625\\
6	9.63606375979495\\
};
\addplot [color=black,mark size=0.5pt,only marks,mark=*,mark options={solid,fill=black,draw=black},forget plot]
  table[row sep=crcr]{%
7	7.11719575917047\\
7	7.15622734246805\\
7	7.35803455085558\\
7	7.76847895688258\\
};
\end{axis}
\end{tikzpicture}%
%
 } \hfil
        \subfigure{%
    \input{boxplot-body-generalisationMaxError.tikz}%
 } \hfil
        \subfigure{%
    \input{boxplot-body-maxErrorSparseFitting.tikz}%
}
     }
     \vspace{-5mm}
    \caption{Boxplots of the quantitative measures in each column for the body shapes dataset. Quantitatively all methods have comparable performance, apart from ICA which performs worse.} 
    \label{body-boxplots}
\end{figure*}
\begin{figure}%
     \centerline{\includegraphics[scale=0.25]{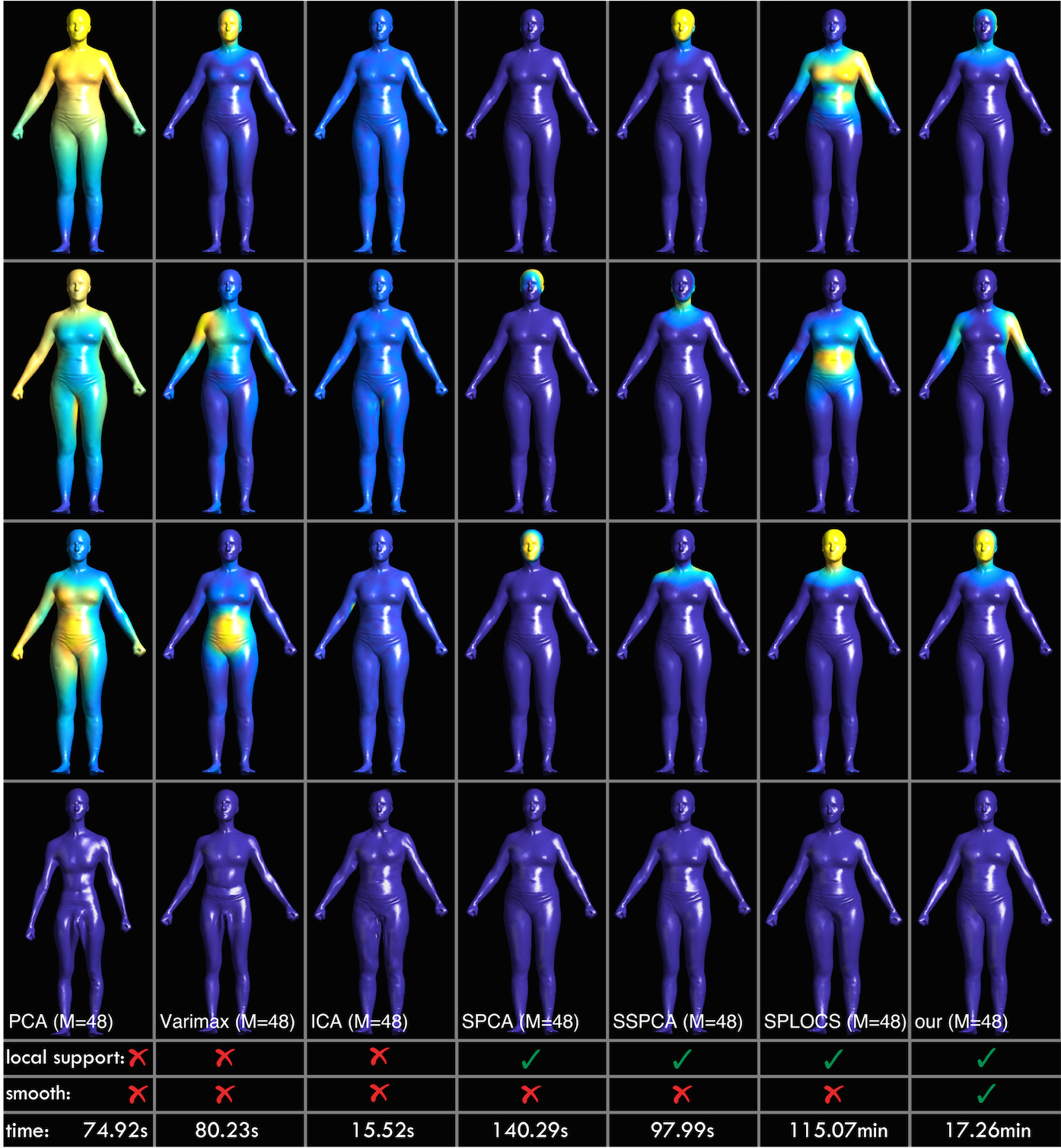}}
      \vspace{-2mm}
     \caption{The deformation magnitudes reveal that SPCA, SSPCA, SPLOCS and our method obtain local support factors (in the second factor of our method the connection is at the back). The bottom row depicts randomly drawn shapes, where only the methods with local support deformation factors result in plausible shapes.}
     \label{bodyDefFactors}
\end{figure}
Our second experiment is based on $1531$ female human body shapes \cite{Yang:2014tj}, where each shape comprises $12500$ vertices that are in correspondence among the training set. Due to the large number of training data and the high level of details in the meshes, we directly factorise the data matrix $\matX$.
The edge set $\mathcal{E}$ now contains the edges of the triangle mesh topology and the weights for edge $e=(i,j) \in \mathcal{E}$ are given by $\vecomega_e = \exp(- (\frac{(\matD_{\text{euc}})_{ij}}{\bar d})^2 )$, where $\bar d$ denotes the average vertex distance between neighbour vertices. Edges with weights lower than $\theta {=} 0.1$ are ignored.

\subsubsection{Results}
Quantitatively the evaluated methods have comparable performance, with the exception that ICA has worse overall performance (Fig.~\ref{body-boxplots}). The most noticeable difference between the methods is the specificity error, where SPCA, SSPCA and our method perform best. 
Fig.~\ref{bodyDefFactors} reveals that SPCA, SSPCA, SPLOCS and our method obtain factors with local support. Apparently, for large datasets, sparsity alone, as used in SPCA and SSPCA, is sufficient to obtain local support factors. However, our method is the only one that explicitly aims for smoothness of the factors, which leads to more realistic deformations, as shown in Fig.~\ref{perfactorDefs}. %

\begin{figure}%
     \centerline{\includegraphics[scale=0.28]{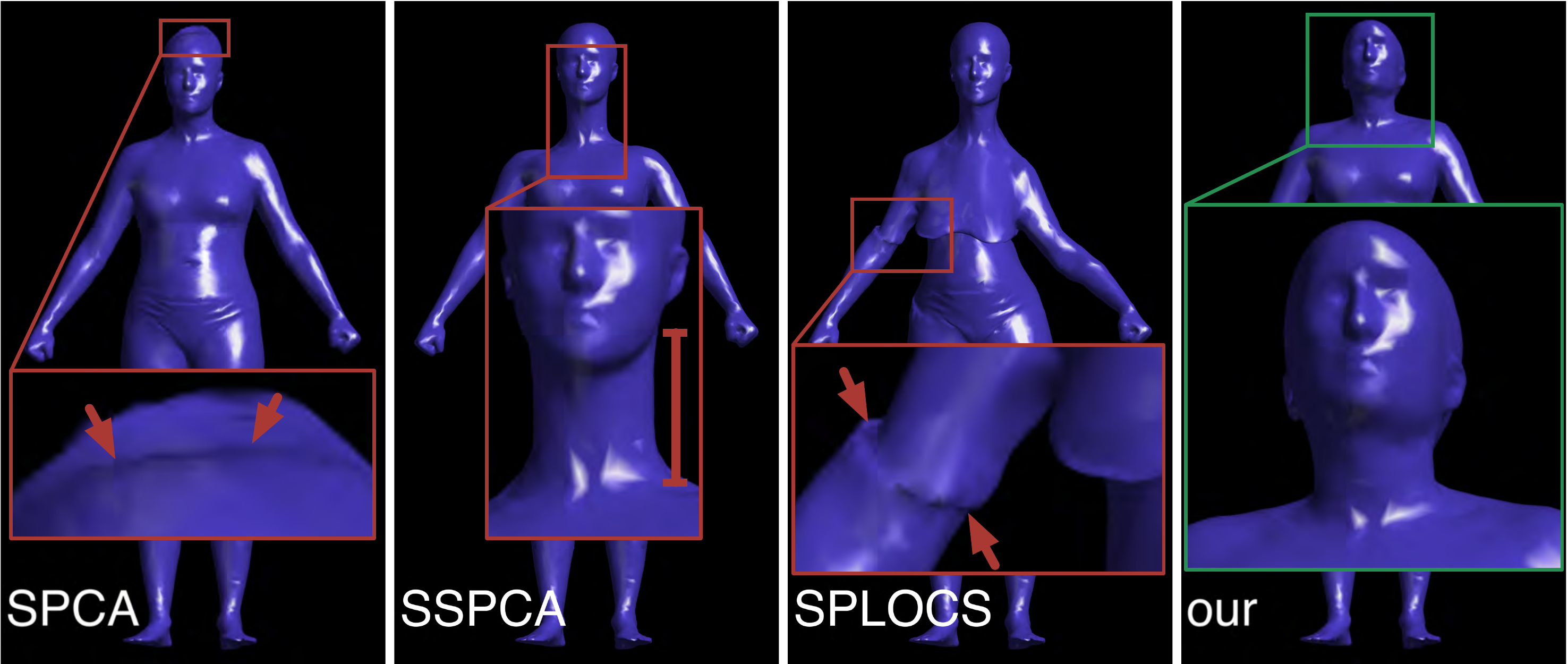}}
      \vspace{-2mm}
     \caption{Shapes ${\bar{\vecx} {-} 1.5 \matPhi_m}$ for SPCA ($m{=}1$), SSPCA ($m{=}3$), SPLOCS ($m{=}1$) and our ($m{=}1$) method (cf.~Fig.~\ref{bodyDefFactors}). Our method delivers the most realistic per-factor deformations.}
        \label{perfactorDefs}
\end{figure}

\section{Conclusion}
We presented a novel approach for learning a linear deformation model from training shapes, where the resulting factors exhibit local support. By embedding sparsity and smoothness regularisers into a theoretically well-grounded matrix factorisation framework, we model local support regions \emph{implicitly}, and thus get rid of the initialisation of the size and location of local support regions, which so far has been necessary in existing methods. %
On the small brain shapes dataset that contains relatively simple shapes, our method improves the state of the art with respect to generalisation and sparse reconstruction. For the large body shapes dataset containing more complex shapes, quantitatively our method is on par with existing methods, whilst it delivers more realistic per-factor deformations.
Since articulated motions violate our smoothness assumption, our method cannot handle them.
However, when smooth deformations are a reasonable assumption, our method offers a higher flexibility and better interpretability compared to existing methods, whilst at the same time delivering more realistic deformations.

\ifcvprfinal
\subsection*{Acknowledgements}
We thank Yipin Yang and colleagues for making the human body shapes dataset publicly available; Benjamin D. Haeffele and Ren\'{e} Vidal for providing their code; Thomas B\"{u}hler and Daniel Cremers for helpful comments on the manuscript; Luis Salamanca for helping to improve our figures; and Michel Antunes and Djamila Aouada for pointing out relevant literature. 
The authors gratefully acknowledge the financial support by the Fonds National de la Recherche, Luxembourg (6538106, 8864515).
\fi

{\small
\bibliographystyle{ieee}
\bibliography{refs}
}

\appendix

\twocolumn[\section{Linear Shape Deformation Models with Local Support using Graph-based Structured Matrix Factorisation -- Supplementary Material}]
\subsection{The matrix $\matE$ in eq.~\eqref{matE}}

The matrix  $\matE \in \R^{3\vert \mathcal{E} \vert \times 3N}$ is defined
as 
    $\matE = \matI_3 \otimes \matE'$,
where $\matE' \in \R^{\vert \mathcal{E} \vert \times N}$ is the weighted incidence matrix of the graph $\mathcal{G} {=} (\mathcal{V}, \mathcal{E}, \vecomega)$ with elements
\begin{align}
  \matE'_{pq} = \begin{cases}
    \sqrt{\vecomega_{e_p}} & \text{if } q = i \text{ for } e_p = (i,j)\\
    -\sqrt{\vecomega_{e_p}} & \text{if } q = j \text{ for } e_p = (i,j)\\
    0 & \text{otherwise} \,.
  \end{cases}
\end{align}
For $p=1,\ldots, \vert \mathcal{E} \vert$, $e_p \in \mathcal{E}$ denotes the $p$-th edge.

\subsection{Proximal Operators}
\begin{mydef}\label{defProx}
  The \emph{proximal operator} (or \emph{proximal mapping}) $\prox_{s\theta}(y): \R^n \rightarrow \R^n$ of a lower semicontinuous function $\theta: \R^n \rightarrow \R$ scaled by $s > 0$ is defined by
\begin{align}
  \prox_{s\theta}(\vecy) = \underset{\vecx}{\arg\min}~\{  \frac{1}{2s} \|\vecx-\vecy\|_2^2 + \theta(\vecx)  \}\,.
\end{align}
\end{mydef}

\subsubsection{Proximal Operator of $\|\cdot\|_A$}

The proximal operator $\prox_{\|\cdot\|_A}(\vecy)$ of $$\|\cdot\|_A = \lambda_A \| \cdot \|_2\,,$$ cf.~eq.~\eqref{normA}, can be computed by the so-called 
\emph{block soft thresholding} \cite{Parikh:2013vb}, i.e.
\begin{align}\label{blocksoftthresh}
   \prox_{ \lambda_A \| \cdot\|_2}(\vecy) = (1-  \lambda_A/\|\vecy\|_2)_+ \vecy \,,
\end{align}
where for a vector $\vecx \in \R^p$, $(\vecx)_+$ replaces each negative element in $\vecx$ with $0$.
As such, a very efficient way for computing the proximal mapping of $\|\cdot\|_A$ is available.

\subsubsection{Proximal Operator of  $\|\cdot\|_\Phi$}
The proximal mapping of 
\begin{align}
  \|\cdot\|_\Phi = 
&\lambda_1 \| \cdot \|_1 + 
   \lambda_2 \| \cdot \|_2  \nonumber\\
   &\quad + \lambda_\infty \| \cdot \|_{1,\infty}^\mathcal{H} + 
    \lambda_{\mathcal{G}} \| \matE \cdot \|_2 \nonumber\,,
\end{align}
cf.~eq.~\eqref{normPhi}, is given by
\begin{align}
  &\prox_{\|\cdot\|_\Phi}(\vecz)  = \nonumber\\ 
  & \label{proxPhi} \quad  \underset{\vecx}{\arg\min}~\{  \frac{1}{2} \|\vecx-\vecz\|_2^2 + \lambda_1 \| \vecx \|_1   + \lambda_2 \| \vecx \|_2 \\ 
   &    \quad\qquad\qquad 
  +  \lambda_{\infty} \| \vecx \|_{1,\infty}^\mathcal{H} + 
    \lambda_{\mathcal{G}} \| \matE \vecx \|_2   \}\ \,.\nonumber 
\end{align}

Due to the non-separability caused by the linear mapping $\matE$ inside the $\ell_2$ norm, this case is more difficult compared to $\prox_{\|\cdot\|_A}(\vecy)$. However, by introducing
    \begin{align}
      &{f(\vecx) = \lambda_1 \| \vecx \|_1 +  \lambda_2 \| \vecx \|_2 + 
      \lambda_{\infty} \| \vecx \|_{1,\infty}^\mathcal{H}} %
      \intertext{and}
      & {g(\vecx) = \lambda_{\mathcal{G}} \|\vecx \|_2}\,,
\end{align}
eq.~\eqref{proxPhi} can be rewritten as
\begin{align}\label{proxPhiDFB}
    \underset{\vecx}{\arg\min}~\{  \frac{1}{2} \|\vecx-\vecz\|_2^2 +  f(\vecx) + g(\matE\vecx)\}\,.
\end{align}
In this form, \eqref{proxPhiDFB} can now be solved by a dual forward-backward splitting procedure \cite{Combettes:2010uu,Combettes:2009wd} as shown in algorithm~\ref{proxPhiCode}, which we will explain in the rest of this section. The efficient computability hinges on the efficient computation of the (individual) proximal mappings of $f$ and $g$.

\begin{algorithm}\label{proxPhiCode}
\scriptsize
\SetKwInput{Input}{Input}
\SetKwInput{Initialise}{Initialise}
\SetKwInput{Output}{Output}
\SetKwInput{Parameters}{Parameters}
\DontPrintSemicolon

 \Input{$\vecz \in \R^{3N}$}
 \Output{$\vecx = \prox_{\|\cdot\|_\Phi}(\vecz)$}
 \Parameters{$\lambda, ~\lambda_1, ~\lambda_{\infty},~\lambda_2,~\lambda_{\mathcal{G}} $}
\Initialise{$\vecv, \epsilon \leftarrow 1e^{-4}, \gamma \leftarrow 1.999, \beta = \frac{1+\epsilon}{2}$}
 \tcp{Normalise $\matE$ (homogeneity of norm)}
 $\lambda_{\mathcal{G}} \leftarrow  \lambda_{\mathcal{G}} \|\matE\|_F; \quad \matE \leftarrow \frac{\matE}{\|\matE\|_F}$\\

 \Repeat{convergence}{
 $\vecx \leftarrow \prox_{\lambda_1 \| \cdot\|_1}(\vecz - \matE^T \vecv)$~~\tcp{$\ell_1$ prox, \eqref{softthresh}}

 $\vecx \leftarrow \prox_{\lambda_{\infty} \| \cdot\|_{1,\infty}^\mathcal{H}}(\vecx)$~~\tcp{$\ell_1/\ell_\infty$, \cite{Bach:2011ty,Duchi:2008vv}}

 $\vecx \leftarrow \prox_{\lambda_{2} \| \cdot\|_{2}}(\vecx)$~~\tcp{$\ell_2$ prox, \eqref{blocksoftthresh}}

 $\vecv \leftarrow \vecv {+}  \beta\gamma (\matE \vecx {-} \prox_{ \lambda_{\mathcal{G}} / \gamma \|\cdot\|_2}(\frac{\vecv + \gamma \matE \vecx}{\gamma} ) )$ \tcp{$^\dagger$}

 }

 \caption{Dual forward-backward splitting algorithm to compute the proximal mapping of $\|\cdot\|_\Phi$.}
\end{algorithm}

Since $g$ is a (weighted) $\ell_2$ norm, its proximal mapping is given by \emph{block soft thresholding} presented in eq.~\eqref{blocksoftthresh}.

In $f$, the sum of weighted $\ell_1$ and $\ell_1/\ell_\infty$ norms is  a term that appears in the \emph{sparse group lasso} and can be computed by applying the \emph{soft thresholding} \cite{Parikh:2013vb} %
\begin{align}\label{softthresh}
  \prox_{s \| \cdot\|_1}(\vecy) = (\vecy -s)_+ - (-\vecy - s)_+
\end{align}
first, followed by \emph{group soft thresholding} \cite{Bach:2011ty}. As shown in \cite[Thm. 3]{Haeffele:2014wj}, the $\ell_2$ term can be additionally incorporated by composition, i.e.~by subsequently applying \emph{block soft thresholding} as presented in \eqref{blocksoftthresh}.

Now, to solve the \emph{group soft thresholding} for the proximal mapping of the $\ell_1/\ell_\infty$ norm, one can use the fact that for any norm $\omega$ with dual norm $\omega^*$, $\prox_{s\omega}(\vecy) = {\vecy - \proj_{\omega^* \leq s}(\vecy)}$. By $\proj_{\omega^* \leq s}(\vecy)$, we denote the projection of $\vecy$ onto the $\omega^*$ norm ball with radius $s$ \cite{Bach:2011ty}. The dual norm of the $\ell_1/\ell_\infty$ norm, eq.~\eqref{l1linf}, is the $\ell_\infty/\ell_1$ norm
\begin{align}
    \| \vecz \|_{\infty,1}^\mathcal{H} =  \max_{g \in \mathcal{H}}\| \vecz_g \|_{1} \,.
\end{align}
The orthogonal projection $\proj_{\| \cdot \|_{\infty,1}^\mathcal{H} \leq s}$ onto the $\ell_\infty/\ell_1$ ball is obtained by projecting separately each subvector $\vecz_g$ onto the $\ell_1$ ball in $\R^{\vert g \vert}$ \cite{Bach:2011ty}. This demands an efficient projection onto the $\ell_1$ norm ball. Due to the special structure of our groups in $\mathcal{H}$, i.e.~there are exactly $N$ non-overlapping groups, each of which consisting of three elements, a vectorised Matlab implementation of the method by Duchi et al.\ \cite{Duchi:2008vv} can be employed.

\begin{mydef} (Convex Conjugate)\\
Let $\theta^\star(\vecy) = \sup_{\vecx}(\vecy^T\vecx - \theta(\vecx))$ be the \emph{convex conjugate} of $\theta$.  
\end{mydef}

The update of the dual variable $\vecv$ in algorithm~\ref{proxPhiCode} (see the line marked with $^\dagger$) is based on the update
\begin{align}
  \vecv \leftarrow \vecv + \beta (\prox_{\gamma (\lambda_{\mathcal{G}}  \|\cdot\|_2)^\star}(\vecv + \gamma \matE \vecx) - \vecv ) \,,\label{dualupdate}
\end{align}
presented in \cite{Combettes:2009wd,Combettes:2010uu}, where $(\cdot)^\star$ denotes the convex conjugate.

Let us introduce some tools first.
\begin{lemma} \label{moreau} (Extended Moreau Decomposition) \cite{Parikh:2013vb}\\
It holds that
\begin{align}
    \forall \vecy: ~\vecy = \prox_{s\theta}(\vecy) + s\prox_{\theta^\star/s}(\vecy/s) \,.
  \end{align}  
\end{lemma}
\begin{lemma} (Conjugate of conjugate) \label{conjconj} \cite{Boyd:2009kl}\\
  For closed convex $\theta$, it holds that $\theta^{\star\star} = \theta$.
\end{lemma} 
\begin{corollary}\label{moreauCorol}
  For convex and closed $\theta$, it holds that
\begin{align}
    \forall \vecy: ~\vecy = \prox_{s\theta^\star}(\vecy) + s\prox_{\theta/s}(\vecy/s) \,.
  \end{align} 
\end{corollary}
\begin{proof}
  Define $\theta' = \theta^\star$ and apply Lemma \ref{moreau} and Lemma \ref{conjconj} with $\theta'$ in place of $\theta$. 
\end{proof}

Since for our choice of $\theta = \lambda_{\mathcal{G}}  \|\cdot\|_2$, $\theta$ is a closed convex function, by  Corollary \ref{moreauCorol}, one can write
\begin{align}
  \prox_{\gamma (\lambda_{\mathcal{G}}  \|\cdot\|_2)^\star}(\vecy) = 
  \vecy - \gamma \prox_{ (\lambda_{\mathcal{G}} / \gamma) \|\cdot\|_2}(\vecy/\gamma) \,.
\end{align}
With that, the right-hand side of eq.~\eqref{dualupdate} can be written as
\begin{align}
&  \vecv + \beta (\vecv + \gamma \matE \vecx - \gamma\prox_{\lambda_{\mathcal{G}}  / \gamma  \|\cdot\|_2}(\frac{\vecv + \gamma \matE \vecx}{\gamma} ) - \vecv ) = \nonumber\\
& \vecv +  \beta\gamma (\matE \vecx - \prox_{\lambda_{\mathcal{G}}  / \gamma \|\cdot\|_2}(\frac{\vecv + \gamma \matE \vecx}{\gamma} ) ) \,.
\end{align}

\subsection{Computational Complexity}
In practice it holds that $N \gg \max(M,K)$.
The gradient steps for the updates of $\matPhi$ and $\matA$ have both time complexity 
$\mathcal{O}(N \max(M^2,K^2))$, 
the proximal step for $\matA$ has complexity $\mathcal{O}(MK)$, and the proximal step for $\matPhi$ has complexity $\mathcal{O}(MN\vert\mathcal{E}\vert  n_\text{it})$, where $n_\text{it}$ is the number of iterations for the dual forward-backward splitting procedure (we used a maximum of $n_\text{it}{=}20$). Thus, the total time complexity for one iteration in algorithm~\ref{bcdCode} is $\mathcal{O}(N \max(M^2,K^2) {+} MN \vert\mathcal{E}\vert  n_\text{it})$. Since in practice the number of vertices $N$ and the number of edges in the graph $\vert\mathcal{E}\vert$ are larger than $M$ and $K$, the runtime complexity is dominated by the proximal step for $\matPhi$, which in our experiments takes around $60\%$ of the time for the brain shapes dataset ($N{=}1792,~M{=}96,~K{=}17~,\vert\mathcal{E}\vert{=}19182$), and uses more than $90\%$ of the time for the human body shapes dataset ($N{=}12500,~M{=}48,~K{=}1531,\vert\mathcal{E}\vert{=}99894$).

\subsection{Factor Splitting}
\noindent The \emph{factor splitting} procedure is presented in algorithm~\ref{factorSplitting}.

\begin{algorithm}\label{factorSplitting}
\scriptsize
\SetKwInput{Input}{Input}
\SetKwInput{Initialise}{Initialise}
\SetKwInput{Output}{Output}
\SetKwInput{Parameters}{Parameters}
\DontPrintSemicolon

 \Input{factor $\vecphi \in \R^{N \times 3} \text{ where } \VEC(\vecphi) =  \matPhi_m$, graph $\mathcal{G} = (\mathcal{V},\mathcal{E})$}
 \Output{$J$ factors $\matPhi' \in \R^{3N \times J}$ with local support}
\Initialise{$\mathcal{E'} = \emptyset$, $\matPhi' = [\,]$}
 \tcp{build activation graph $\mathcal{G'}$}
 \ForEach{$(i,j) = e \in \mathcal{E}$}{
    \If( \tcp*[h]{vertex $i$ or $j$ is active}){$\vecphi_{i,:} \neq \zerovec \vee \vecphi_{j,:} \neq \zerovec $ }{
      $\mathcal{E'} = \mathcal{E'} \cup \{e\}$ \tcp{add edge}
    }
 }
 $\mathcal{G'} = (\mathcal{V}, \mathcal{E'})$\\
 \BlankLine
 \tcp{find connected components \cite{Tarjan:1972hk}}
 $\mathcal{C}$ = connectedComponents($\mathcal{G'}$) \tcp*{$\mathcal{C} \subset 2^{\mathcal{V}}$}
 \BlankLine
 \ForEach( \tcp*[h]{add new factor for $c \subset \mathcal{V}$}){$c \in \mathcal{C}$}{
    $\vecphi' = \zerovec_{N \times 3}$\\
    $\vecphi'_{c,:} = \vecphi_{c,:}$\\
    $\matPhi' = [\matPhi',~ \VEC(\vecphi')]$
 }
 \caption{Factor splitting procedure.}
\end{algorithm}

\subsection{Convergence Plots}
For $100$ different initialisations (cf.~section \ref{expres}), the convergence plots are shown for both datasets in Fig.~\ref{convplot}. It can be seen that the main convergence occurs after around $10$ iterations. Moreover, for all $100$ initialisations the objective value are near-congruent.

\begin{figure}[h]
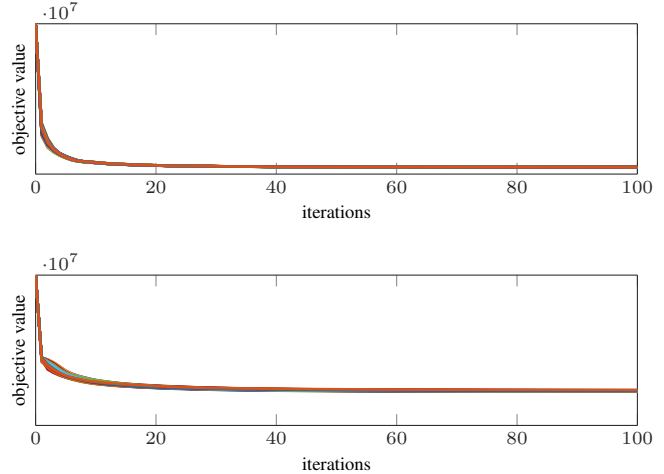
%
     \centerline{
        \subfigure{%
    \input{brain_convergenceplot.tikz}%
 }
     } 
     \centerline{
        \subfigure{%
    \input{SPRING_convergenceplot.tikz}%
 }
     } 
    \caption{Convergence plots for the brain shapes dataset (top) and the human body shapes dataset (bottom). The iterations are shown on the horizontal axis and the (relative) objective value in eq.~\eqref{optProb} is shown on the vertical axis. Note that in each subfigure all $100$ lines are near-congruent.} 
    \label{convplot}
\end{figure}

\subsection{Parameter Random Sampling}

The methods kPCA, SPCA, SSPCA, SPLOCS and our method require various parameters to be set. In order to find a good parametrisation we conducted random sampling over the parameter space, where we determined reasonable ranges for the parameters experimentally. 

In Table \ref{tab:paramSample} the distributions and default values of the parameters for each method are given. $\mathcal{U}(a,b)$ is the uniform distribution with the open interval $(a,b)$ as support, $10^{\mathcal{U}(\cdot,\cdot)}$ is a distribution of the random variable $y = 10^x$, where $x \sim \mathcal{U}(\cdot,\cdot)$,
and $\bar{d}_{\max}$ is the largest distance between all pairs of vertices of the mean shape $\bar{X}$. %
\begin{table}[htbp]
\setlength{\tabcolsep}{1mm}
  \centering
  \caption{Assumed distributions of the parameters interpreted as random variables.}
    \begin{tabular}{lp{6.9cm}}
    \toprule
        \textbf{Method} & \textbf{Parameter Distribution / Default Value}\\
    \midrule
    kPCA & $\beta \sim (\mathcal{U}(1,10))^{-1}$ \\

    SPCA  & ${\lambda \sim \frac{1}{3N}10^{\mathcal{U}(-4,-3)}}$ (see eq.~(2) in \cite{Jenatton:2009tq}) \\ 

    SSPCA & ${\lambda \sim \frac{1}{N}10^{\mathcal{U}(-4,-3)}}$ (see eq.~(2) in \cite{Jenatton:2009tq}) \\ 

    SPLOCS & {${\lambda \sim (3K){\cdot} 10^{\mathcal{U}(-4,-3)}};$
    ${d_{\min} \sim c {-} w};$
    ${d_{\max} \sim c {+} w}$}\\
    & (see eq.~(6) in \cite{Neumann:2013gb}), where ${c \sim \mathcal{U}(0.1,\bar{d}_{\max} - 0.1)}\qquad$ and
    ${w \sim \mathcal{U}(0, \min(\vert c - 0.1\vert,\vert \bar{d}_{\max} -0.1 - c\vert))}$ \\ 

    our & ${\beta \sim (\mathcal{U}(1,10))^{-1}};$
    {${ \lambda = 64 {\cdot} \frac{3NK }{M}};$}
    {${ \lambda_{\mathcal{G}} = \frac{1}{\sqrt{3 \vert \mathcal{E}\vert}}};$ 
    ${ \lambda_1 = \frac{1}{\sqrt{3N}}};$}
    ${ \lambda_2 = \frac{1}{\sqrt{3N}}};$
    ${ \lambda_\infty   = 2 {\cdot} \frac{1}{\sqrt{N}}};$
    ${ \lambda_A = 10^{-4} {\cdot} \frac{1}{ \sqrt{K}}}$\\

    \bottomrule
    \end{tabular}%
    \vspace{-3mm}
  \label{tab:paramSample}%
\end{table}%

For each of the $n_r$ random samples (we set $n_r = 500$ for the brain shapes dataset, and $n_r = 50$ for the human body shapes, due to the large size of the dataset) we compute the $8$ scores described in section \ref{quanmeas} and store them in the score matrix $\matS \in \R^{n_r \times 8}$. 
After (linearly) mapping the elements of each $n_r$-dimensional column vector in $\matS$ onto the interval $[0,1]$, the best parametrisation is determined by finding the index of the smallest value of the vector $\matS \onevec_{8} \in \R^{n_r}$. %

For the lambda parameters of the proposed method we identified default values that we used for the evaluation of both datasets. Moreover, a normalisation of the parameters is conducted for all methods. For kPCA, SPCA, SSPCA and our method the random sampling was conducted only for a single parameter, whereas for the SPLOCS method three parameters had to be set, two of them related to the size of the local support region.

\end{document}